\documentclass{article}

\usepackage{balance} 

\usepackage[margin=3cm]{geometry}

\usepackage{amsmath,amssymb,amsthm,mathtools}
\usepackage[utf8]{inputenc}
\usepackage{xspace}
\usepackage{url}
\usepackage{dsfont}
\usepackage{enumerate}
\usepackage{rotating}
\usepackage{lscape}
\usepackage{hyperref}
\usepackage{booktabs}
\usepackage[algo2e,ruled,vlined,linesnumbered]{algorithm2e}

\usepackage{caption}
\usepackage{subcaption}

\usepackage{balance}
\usepackage[shortcuts]{extdash}

\usepackage{enumitem}

\usepackage[numbers,longnamesfirst]{natbib}

\newtheorem{theorem}{Theorem}
 \newtheorem{definition}[theorem]{Definition}
 
 \newtheorem{lemma}[theorem]{Lemma}
 
 \newtheorem{claim}[theorem]{Claim}

 \numberwithin{theorem}{section}

\usepackage{tikz}
\usetikzlibrary{shapes,arrows}
\usetikzlibrary{decorations}
\usetikzlibrary{plotmarks}
\usetikzlibrary{calc}
\usetikzlibrary{mindmap}
\usetikzlibrary{shadows}
\usetikzlibrary{backgrounds}
\usetikzlibrary{shapes}
\usetikzlibrary{shapes.symbols}

\usepackage{pgfplots}

\newcommand{\ones}[1]{|#1|_1}
\renewcommand{\Pr}{P}

\newcommand{\hv}{\ensuremath{\mathrm{HV}}\xspace} 
\newcommand{\OMM}{\textsc{OMM}\xspace}
\newcommand{\COCZ}{\textsc{COCZ}\xspace}
\newcommand{\LOTZ}{\textsc{LOTZ}\xspace}
\newcommand{\mLOTZ}{\ensuremath{m}\text{-}\textsc{LOTZ}\xspace}

\newcommand{\mCOCZ}{\ensuremath{m}\text{-}\textsc{COCZ}\xspace}
\newcommand*{\E}{\mathrm{E}}

\DeclareMathOperator{\Unif}{Unif}                         

\newcommand{\LO}{\textsc{LO}\xspace}                      
\newcommand{\TZ}{\textsc{TZ}\xspace}

\newcommand{\LOTZFULL}{\textsc{LeadingOnesTrailingZeroes}\xspace}
\newcommand{\mLOTZFULL}{\ensuremath{m}\text{-}\textsc{LeadingOnesTrailingZeroes}\xspace}
\newcommand{\OMMFULL}{\textsc{OneMinMax}\xspace}
\newcommand{\COCZFULL}{\textsc{CountingOnesCountingZeros}\xspace}

\newcommand{\OJZJk}{$\ensuremath{m}\text{-}\textsc{OJZJ}_k\xspace$}

\newcommand{\ojzjk}{$\textsc{OJZJ}_k\xspace$}

\newcommand{\mplEA}[2]{$(#1+#2)$-EA\xspace} 
\newcommand{\opoEA}[1]{\mplEA{1}{1}} 

\newcommand{\PAES}{PAES-25\xspace}

\newenvironment{proofofclaim}{\textsc{Proof of Claim.}}{\hfill\scriptsize\scalebox{1}{$\blacksquare$}}

\newcommand{\hide}[1]{}

\usepackage{xcolor}

\definecolor{todocolor}{rgb}{0.9,0.1,0.1}
\definecolor{changedcolor}{rgb}{0.42,0.27,0.57}
\definecolor{addedcolor}{rgb}{0.867,0.176,0.361}

\author{
  Andre Opris\\
  University of Passau\\
  Passau, Germany
  }

\title{A First Runtime Analysis of the PAES-25: An Enhanced Variant of the Pareto Archived Evolution Strategy}

\begin{document}

\maketitle

\begin{abstract}
This paper presents a first mathematical runtime analysis of \PAES, an enhanced version of the original Pareto Archived Evolution Strategy (PAES) coming from the study of telecommunication problems over two decades ago to understand the dynamics of local search of MOEAs on many-objective fitness landscapes. We derive tight expected runtime bounds of \PAES with one-bit mutation on \mLOTZFULL until the entire Pareto front is found: $\Theta(n^3)$ iterations if $m=2$, $\Theta(n^3 \log^2(n))$ iterations if $m=4$ and $\Theta(n \cdot (2n/m)^{m/2} \log(n/m))$ iterations if $m>4$ where $n$ is the problem size and $m$ the number of objectives. To the best of our knowledge, these are the first known tight runtime bounds for an MOEA outperforming the best known upper bound of $O(n^{m+1})$ for (G)SEMO on \mLOTZ when $m \geq 4$. We also show that \emph{archivers}, such as the \emph{Adaptive Grid Archiver} (AGA), \emph{Hypervolume Archiver} (HVA) or \emph{Multi-Level Grid Archiver} (MGA), help to distribute the set of solutions across the Pareto front of \mLOTZ efficiently. We also show that \PAES with standard bit mutation optimizes the bi-objective \LOTZFULL benchmark in expected $O(n^4)$ iterations, and we discuss its limitations on other benchmarks such as \OMMFULL or \COCZFULL.

\end{abstract}


\maketitle

\section{Introduction}
\label{sec:introduction}

In many real-world scenarios, optimization involves handling multiple objectives that often conflict with each other. A common approach to tackle such multi-objective optimization problems is to generate a diverse set of Pareto-optimal solutions. These are solutions where no objective can be improved without degrading at least one other. A decision maker can then select the most suitable trade-off based on their preferences. Owing to their population-based nature, multi-objective evolutionary algorithms (MOEAs) have become essential tools and are widely applied across various domains in practice~\cite{kdeb01,coello2013evolutionary,9515233,LUUKKONEN2023102537}. 

The mathematical runtime analysis of MOEAs began roughly two decades ago~\cite{LaumannsTZWD02,Giel03,Thierens03}. In recent years, this field has experienced significant growth, with studies focusing on well-known algorithms such as (G)SEMO, NSGA-II, NSGA-III, SMS-EMOA, and SPEA2~\cite{Laumanns2004,DoerrNearTight,UpBian,RenBLQ24,Opris2025,OprisNSGAIII}. Additionally, research has explored how MOEAs can be applied to combinatorial optimization problems like minimum spanning trees~\cite{DBLP:conf/ijcai/CerfDHKW23}, subset selection problems~\cite{MOEASubset} or submodular optimization~\cite{QianYTYZ19}. A useful dichotometry of how to analyze such algorithms theoretically can be found in~\cite{Dang-Levels}.

Another notable MOEA is the \emph{Pareto Archived Evolution Strategy} (PAES) algorithm, originally proposed by~\citet{781913} coming from the study of telecommunication problems. PAES is a (1+1) evolutionary strategy that performs mutation on a single solution, referred to as the \emph{current solution}. To identify non-dominated solutions, the algorithm maintains an archive of previously found non-dominated solutions, enabling it to obtain many different trade-offs. Its size can be limited by a threshold $L$ smaller than the Pareto front. When the archive is full, so-called \emph{archivers} decide whether to accept mutated solutions as the new current solution and which solution to be removed in order to preserve diversity. As this algorithm maintains and mutates a single current solution instead of selecting a solution from a population, it is very simple and particularly well-suited for solving problems relying on multi-objective local search. For instance, in~\cite{PAESFIRST} it is proven that this algorithm using one-bit mutation beats SEMO on a problem called \textsc{Path} by a factor of $n$. A further key advantage of PAES is its potential to scale better to many-objective problems compared to traditional MOEAs. This is primarily because it does not rely on selective pressure, which tends to deteriorate as the number of objectives increases and the Pareto front grows exponentially~\cite{OprisNSGAIII,Zheng_Doerr_2024,Opris2025,DoerrNearTight}. Due to its strong empirical performance across a variety of benchmark problems~\cite{781913}, PAES has gained significant traction (\cite{781913} and~\cite{Knowles2000} together have around~5000 citations). 

However, PAES discards a mutated current solution if the archive already contains such a solution with the same fitness, which leads to some limitations. In~\cite{PAESFIRST} it is rigorously shown that it fails with overwhelming probability to find the entire Pareto front of the \LOTZ benchmark, an easy pseudo-Boolean function. Additionally, it is shown in~\cite{781913} that it performs empirically poorly on more complex pseudo-Boolean functions like LITZ, sLITZ and FRITZ, alternative multi-objective generalizations of the \LOTZ benchmark. To address these shortcomings, \citet{Knowles2025} proposed an enhanced version of PAES, called \PAES. This updated algorithm has slightly different acceptance criteria. It adds solutions to the archive and accepts mutated current solutions as the new current solution, even if one with the same fitness already exists in the archive which is then  replaced. This enables progress on fitness plateaus. Empirical results by~\citet{Knowles2025} suggest that \PAES not only progresses more quickly towards the Pareto front but also discovers a significantly larger fraction of the front on benchmarks such as \LOTZ, LITZ, sLITZ and FRITZ compared to the standard PAES.

Despite these practical results, a rigorous theoretical understanding of PAES remains limited. Apart from~\cite{PAESFIRST}, no formal runtime analysis existed for either the original PAES or its updated variant \PAES. Consequently, it is still unclear under which conditions \PAES performs well, and what structural features of problems contribute to its success. Addressing these open questions could offer valuable guidance for designing more efficient variants of \PAES and advancing its applicability in practice.

\textbf{Our contribution:} In this paper we make a significant step towards such an understanding and provide the first runtime analysis of \PAES on different benchmark problems. Our three main contributions are detailed as follows.

Firstly, we analyze \PAES with one bit mutation (i.e. a single bit chosen uniformly at random is flipped) on \mLOTZ. We show that if the archive size is sufficiently large, at least equal to the size of the Pareto front and possibly larger depending on the archiving strategy, then the algorithm efficiently discovers all Pareto optimal search points. Specifically, the number of iterations required is $\Theta(n^3)$ for $m = 2$, $ \Theta(n^3 \log^2 n) $ for $m = 4$, and $ \Theta(n \cdot (2n/m)^{m/2} \log(n/m))$ for $m > 4$. The key insight is that the current solution reaches the Pareto front easily and then performs a random walk on the Pareto front which can be described by a suitable random walk on a grid graph, the product graph of the single path graph. All Pareto-optimal solutions are found if all nodes in the grid graph are visited, and estimating the cover time of product graphs is a well-known task (see for example~\cite{JONASSON2000181}). To the best of our knowledge, these are the first known tight runtime bounds in the many-objective setting and surprisingly, if $ m \geq 4$, \PAES achieves a strictly better asymptotic runtime than the best known upper bound of $O(n^{m+1})$ established for SEMO on \mLOTZ. We also show that if the size of the archive is smaller than the size of the Pareto front, the underlying archiver spreads the solutions in the archive well across the Pareto front in many scenarios where many of them are restricted to \LOTZ, the bi-objective variant of \mLOTZ.

Secondly, we investigate \PAES with standard bit mutation and show that if the size of the archive is at least $n+1$, the archive contains every Pareto-optimal search point of \LOTZ after $O(n^4)$ iterations in expectation. This is another novel approach, since, as far as we know, the only rigorously proven runtime bounds of PAES or variants relies on one-bit mutation.

Thirdly, we want to demonstrate also some limitations of \PAES and rigorously prove that, even if the archive size is $\Omega(n)$, the algorithm can only find $o(n)$ different Pareto-optimal solutions of \OMM and \COCZ after $2^{\Omega(n)}$ iterations. To achieve this, we will use the negative drift theorem from~\cite{Lengler2020Drift}.

\textbf{Related work:} The mathematical runtime analysis of MOEAs began around 20 years ago. Initial work focused on simple algorithms like the simple evolutionary multi-objective optimizer (SEMO)~\cite{LaumannsTZWD02} or the global SEMO (GSEMO)~\cite{Giel03}. Since then, the field has steadily advanced. The foundational runtime results for (G)SEMO were presented in~\cite{LaumannsTZ04}, the journal version of~\cite{LaumannsTZWD02}, analyzing pseudo-Boolean functions like \mLOTZ and \mCOCZ. There it was shown that (G)SEMO can find the entire Pareto front of \mCOCZ (the $m$-objective analogues of the classic \COCZ benchmark) and \mLOTZ in an expected number of $O(n^{m+1})$ fitness evaluations, for an even number of objectives. The upper bound for \LOTZ naturally extends to more than two objectives while the bi-objective \COCZ guarantees a runtime of $O(n^2 \log(n))$. Subsequent work by~\citet{BianQT18ijcaigeneral} improved the results for \mCOCZ to $O(n^m)$ for $m > 4$ and to $O(n^3\log n)$ in the special case $m = 4$. These results were further refined by~\cite{DoerrNearTight}, who showed that for \mCOCZ and \mLOTZ the runtime depends only linearly in the size of the largest set of incomparable solutions, aside from small polynomial factors in the number $m$ of objectives and $n$. That work also analyzed (G)SEMO on the $m$-\OMM and \OJZJk~benchmarks (the $m$-objective variants of the bi-objective \OMM and \ojzjk). Most recently,~\citet{Doerr_Krejca_Rudolph_2025} presented the first runtime analysis of (G)SEMO on integer-valued search spaces and there are even lower runtime bounds of GSEMO on \COCZ,  \OMM and \ojzjk~\cite{doerr2025tightruntimeguaranteesunderstanding}. 

These insights motivated the rigorous analysis of more sophisticated MOEAs that are commonly used in practice. The most prominent among them is NSGA-II~\cite{DebPAM02} which has received over~50000 citations. The first rigorous runtime analysis of this algorithm was conducted by~\citet{ZhengLuiDoerrAAAI22} on classical benchmark problems. This was followed by a series of results investigating its performance on a multimodal problem~\cite{Qu2022PPSN}, approximation qualities in covering the Pareto front~\citep{DoerrApprox}, robustness to noisy environments~\cite{DaOp2023}, lower bounds~\cite{DoerrQu2023a}, the usefulness of crossover~\citep{Dang2024,DoerrQu2022}, trap functions~\cite{DangEfficient2024} and stochastic population update~\cite{UpBian}. NSGA-II has also been studied in the context of combinatorial optimization problems, including the minimum spanning tree~\cite{DBLP:conf/ijcai/CerfDHKW23} and the subset selection problem~\cite{MOEASubset}. Despite these successes, NSGA-II's performance is very limited to the bi-objective setting as shown in~\cite{Zheng2023Inefficiency}. This limitations has recently led to the runtime analysis of other prominent MOEAs designed for many-objective, such as NSGA-III~\cite{WiethegerD23,OprisNSGAIII,Opris2025,opris2025multimodal}, SMS-EMOA~\cite{Zheng_Doerr_2024}, SPEA-2~\cite{RenBLQ24} and a modified version of NSGA-II with a simple tie-breaking rule~\cite{Krejca2025b}. All of these algorithms are population based and, in each iteration, select parents for mutation uniformly at random in contrast to PAES and \PAES. An additional property of \PAES is that it uses an archive to store non-dominated solutions. The study of the effects of such archives in multiobjective optimization was initiated only recently by \citet{ArchiveMOEAs}. They introduced an unbounded archive to drastically reduce the required population size of MOEAs for efficiently optimizing classical benchmark problems, resulting in significant speedups. Their analysis leverages crossover at boundary points to effectively cover the entire Pareto front. 

However, for the archive-based algorithm PAES, to the best of our knowledge only one rigorous runtime analysis is known~\cite{PAESFIRST}. In that work, the pseudo-Boolean function \textsc{Path} is introduced, where PAES with one-bit mutation has an expected runtime bound of $\Theta(n^2)$ which outperforms SEMO by a factor of $n$. However, it is also shown that PAES fails to find the complete Pareto front of the \LOTZ benchmark with overwhelming probability. We are not aware of any other rigorous runtime results for the standard PAES or its variant \PAES. In this paper we show that \PAES is, in fact, able to find the Pareto front of \LOTZ very efficiently. 

\section{Preliminaries}

\emph{Notation:} For a finite set $A$ we denote by $\vert{A}\vert$ its cardinality. 
Let $\log$ be the logarithm to base $e$ and $[k]:=\{1, \ldots , k\}$ for $k \in \mathbb{N}$. The number of ones in a bit string $x$ is denoted by $\ones{x}$. The number of leading ones in $x$, denoted by $\LO(x)$, is the length of the longest prefix of $x$ which contains only ones, and the number of trailing zeros in $x$, denoted by $\TZ(x)$, the length of the longest suffix of $x$ containing only zeros respectively. For example, if $x=1110101100000$, then $\LO(x)=3$ and $\TZ(x)=5$. 
This paper is about many-objective optimisation, particularly the maximisation of a discrete $m$-objective function $f(x):=(f_1(x), \ldots , f_m(x))$ where $f_i:\{0,1\}^n \to \mathbb{N}_0$ for each $i \in [m]$. When $m=2$, the function is also called \emph{bi-objective}. Denote by $f_{\max}$ the maximum possible value of $f$ in one objective, i.e. $f_{\max}:=\max\{f_i(x) \mid x \in \{0,1\}^n, i \in [m]\}$. 
Given two search points $x, y \in \{0, 1\}^n$, $x$ \emph{weakly dominates} $y$, denoted by $x \succeq y$,
if $f_{i}(x) \geq f_{i}(y)$ for all $i \in [m]$, and $x$ \emph{(strictly) dominates} $y$, denoted by $x \succ y$, if one inequality is strict; if neither $x \succeq y$ nor $y \succeq x$ then $x$ and $y$ are \emph{incomparable}. A set $S \subseteq \{0,1\}^n$ is a \emph{set of mutually incomparable solutions} with respect to $f$ if all search points in $S$ are incomparable. Each solution not dominated by any other in $\{0, 1\}^n$ is called \emph{Pareto-optimal}. A mutually incomparable set of these solutions that covers all possible non-dominated fitness values is called a \emph{Pareto(-optimal) set} of $f$. 

We also need the following Theorem from~\cite{theorystochasticdrift} to describe the hitting time of unbiased random walks on the line, with one barrier. We will use this result for describing the dynamics of \PAES with standard bit mutation on \LOTZ below.

\begin{theorem}[Theorem 4.2 in~\cite{theorystochasticdrift}]
\label{thm:unbiased-random-walk}
    Let $n \in \mathbb{N}$ and let $(X_t)_{t \in \mathbb{N}}$ be an integrable random process over $\{0, \ldots , n\}$, and let $T:=\inf\{t \in \mathbb{N} \mid X_t = n\}$. Suppose that there is a $\delta \in \mathbb{R}_+$ such that, for all $t<T$, we have the following conditions (variance, drift).
    \begin{itemize}
        \item[(1)]
        $\text{Var}(X_{t+1}-X_t \mid X_0, \ldots , X_t) \geq \delta$.
        \item[(2)]
        $\E(X_{t+1}-X_t \mid X_0, \ldots , X_t) \geq 0$.
    \end{itemize}
Then $E(T) \leq n^2/\delta$.
\end{theorem}

\emph{Algorithm:} The \PAES~\cite{Knowles2025}, a modification of PAES~\cite{781913,Knowles2000} is displayed in Algorithm~\ref{alg:PAES25}. Initially, a search point $s$ from $\{0,1\}^n$ is sampled uniformly at random, and the archive $A_0$ is set to $\{s\}$. Then, as long as the termination criterion is not met, a copy $s'$ of $s$ is created and mutation is applied on $s'$ (i.e. each bit is flipped uniformly at random with probability $1/n$ in case of standard bit mutation and one bit is flipped chosen uniformly at random in case of one-bit mutation) to generate a new solution $c$. The solution $s$ is referred to as the \emph{current solution}, and $c$ as the \emph{candidate solution}. Then, if the candidate solution $c$ weakly dominates a point in the archive $A_t$, all points weakly dominated by $c$ are removed from $A_t$ and $c$ is added to $A_t$. This stands in strict contrast to the standard PAES which adds $c$ only to the archive and accepts it as a new current solution if it is neither weakly dominated by $s$ nor by any point in the archive. The current solution for the next iteration then becomes $c$. Conversely, if $c$ is dominated by a member of the archive, it is discarded. If $c$ is not weakly dominated by any search point $x$ in the archive $A_t$ (that is, it is incomparable to all search points in $A_t$), it is added to $A_t$ if the archive size is smaller than $L$. Otherwise, if the archive is full, an \emph{archiver} is responsible for deciding whether $c$ will be accepted into $A_t$ and becomes the new current solution. The archiver's primary role is to maintain \emph{diversity} in the archive~\cite{Knowles2025}. If the archiver rejects $c$, the archive $A_t$ remains unchanged, and $c$ is discarded. Otherwise, $c$ is added to the archive, becomes the new current solution $s$ for the next iteration, and another point $y \in A_t$, selected by the archiver, is removed from the archive. Note that the archive always contains pairwise incomparable search points and if $A_t=P$ for a Pareto-optimal set in one iteration, then $A_t=P$ for all future iterations. In this paper we examine three archivers which are also discussed in~\citet{Knowles2025}, but with minor modifications to enable progress on fitness plateaus even when the archive is full. Other archivers, such as \emph{crowding distance} archiver (NSGA-II archiver), the SPEA2 archiver or the $\varepsilon$-dominance archiver, exhibit behavior that closely resembles that of one of the three considered in this paper (see for example~\cite{Archiver2011}).

\begin{algorithm2e}[t]
\caption{\PAES with archive size $L$ and archiver $X$ for maximizing $f:\{0,1\}^n \to \mathbb{N}_0^m$}\label{alg:PAES25}
Initialize $s \sim \Unif(\{0,1\}^n)$\;
Set $t=0$ and create set $A_0=\{s\}$ as an \emph{archive}\;
\While{termination criterion not met}{
    Create $s'$ as a copy of $s$\;
    Create $c$ by mutation on $s'$\;
    \uIf{$c$ weakly dominates any element in $A_t$}
    {
        Remove all $x \in A_{t+1}$ weakly dominated by $c$ \;
        $A_{t+1}=A_t \cup \{c\}$, $s=c$}
    \uElseIf{there is $x \in A_t$ dominating $c$}{
        $A_{t+1}=A_t$\;
    }
    \Else{
        \If{$|A_t| < L$}{
        $A_{t+1}=A_t \cup \{c\}$, $s=c$\;
        }
        \Else{
        \If{$c$ is accepted by the archiver $X$~\label{line:algo-15}}
        {
        Choose $y \in A_t$ according to $X$~\label{line:algo-16}\;
        $A_{t+1}=(A_t \setminus \{y\}) \cup \{c\}$, $s=c$
        }
    }
}
$t=t+1$ 
}
\DontPrintSemicolon
\end{algorithm2e}     

\textbf{Adaptive Grid Archiver~(AGA)}~\cite{Knowles2000}: It covers the search space with a \emph{grid} and associates solutions from the archive to \emph{cells} as follows. Let the values $a \geq f_{\max}$ and $\ell \in \mathbb{N}$ be given. By bisecting the interval $[0,a]$ of each objective axis $\ell$ times into disjoint intervals $[0,a/2^{\ell}), \ldots , [(2^{\ell}-2)a/2^{\ell},(2^{\ell}-1)a/2^{\ell}), [(2^{\ell}-1)a/2^{\ell},a]$, the hypercube $[0,a]^m$ is divided into $2^{\ell \cdot m}$ cells. The location of an objective vector $v=(v_1, \ldots , v_m) \in [0,a]^m$ in the objective space is then described by a binary string of length $2^{\ell \cdot m}$, where the decimal value $0 \leq k<2^\ell-1$ of the $j$-th block of $2^\ell$ elements from the left indicates the interval $I=[ka/2^\ell,(k+1)a/2^\ell) \subset [0,a]$ such that $v_j \in I$. If $k=2^\ell-1$ then $I=[ka/2^\ell,(k+1)a/2^\ell]$. Two solutions $x,y$ belong to the same cell if their corresponding objective values $f(x)$ and $f(y)$ are represented by the same binary string of length $2^{\ell \cdot m}$. The entire grid is formed by these cells. 
This archiver always accepts the candidate solution $c$ in Line~\ref{line:algo-15} in Algorithm~\ref{alg:PAES25} (i.e. if $c$ is incomparable to all points $x \in A_t$ and $A_t$ is full). Then, a solution distinct from $c$ from a grid cell containing the most search points is selected uniformly at random and removed from $A_t$ (in Line~\ref{line:algo-15} in Algorithm~\ref{alg:PAES25}). Ties with respect to such cells are broken uniformly at random.

\textbf{Hypervolume Archiver~(HVA)}~\cite{10.1109/TEVC.2003.810755, KnowlesPhD}: For a set $P \subset \{0,1\}^n$, the \emph{hypervolume} with respect to a fitness function $f\colon \{0,1\}^n \rightarrow \mathbb{R}^m$ and a reference point $h \in \mathbb{R}^m$ (preset by the user at the beginning) is defined as
\begin{align*}
\hv(P) \coloneqq \mathcal{L}\left(\bigcup_{x\in P} \left\{v\in\mathbb{R}^m\mid \text{for all } i \in [m]: h_i \leq v_i \wedge v_i \leq f_i(x)\right\}\right)
\end{align*}
where $\mathcal{L}$ denotes the Lebesgue measure. The \emph{hypervolume contribution} of a point $x \in A_t \cup \{c\}$ is the value $\hv(A_t \cup \{c\}) - \hv((A_t \cup \{c\}) \setminus \{x\})$ (see for example~\cite{AUGER201275,ZitzlerPhD} for theoretical investigations of the hypervolume or~\cite{HYPERALGO} for an efficient algorithm to compute minimal hypervolume contributions for subsets of $\{0,1\}^n$). If $c$ contributes smallest to $\hv(A_t \cup \{c\})$ as the only point from $A_t \cup \{c\}$, it is not accepted. Otherwise, it is added to $A_t$, and a different point with the smallest hypervolume contribution to $\hv(A_t \cup \{c\})$ is removed from $A_t$ where ties are broken randomly.

\textbf{Multi-Level Grid Archiver~(MGA)}~\cite{LAUMANNS2011414, LAUMANNS2011415}: 
The multi-level grid archiver (MGA) can be considered as a hybrid approach, combining elements from both the adaptive grid archiver (AGA) and the $\varepsilon$-Pareto archiver (see, for example,~\cite{Archiver2011} for the latter). 
The MGA uses a hierarchical family of \emph{boxes} with respect to a \emph{coarseness level $b \geq 0$} over the objective space. For an objective vector $v:=(v_1,\ldots , v_m)$ we define the \emph{box index vector} at coarseness level $b$ as $v^{\text{box}}:=(v_1^{\text{box}}, \ldots , v_n^{\text{box}})$ where $v_i^{\text{box}}:=\lfloor{v_i \cdot 2^{-b}}\rfloor$. Then, when a candidate solution $c$ is created, the smallest level $b$ is determined such that two box index vectors of solutions from $A_t \cup \{c\}$ are comparable. The candidate solution $c$ is rejected if its corresponding box index vector is the only one that is weakly dominated by another element from $A_t$ at this level $b$. Otherwise, $c$ is added to $A_t$, and an arbitrary solution $x$ from $A_t$ distinct from $c$ that is weakly dominated by another element from $(A_t \setminus \{x\}) \cup \{c\}$ at level $b$ is removed.

In all three schemes, in case of a tie between the candidate solution $c$ and another solution $z \in A_t$ in the archiver (e.g., when they have the same smallest hypervolume contribution in case of HVA), the candidate solution $c$ is always accepted becoming the new current solution. This approach reflects a common strategy and is useful for exploring plateaus. It is a minor difference to some versions displayed in~\cite{10.1109/TEVC.2003.810755}, where a solution $x$ is not necessarily added to $A_t$ if there is already a point in $A_t$ with the same fitness as $x$. In such a case the candidate solution is rejected.

A major difference to the classical MOEAs studied before where a parent is chosen according to some probability distribution from a \emph{population} is that only the current solution is selected for mutation and this solution is either the current solution of the previous iteration if the candidate solution was rejected, or the solution which is added to the archive becoming the new current solution if the candidate solution was accepted in the previous iteration. 

\section{Runtime Analysis of \PAES with One-Bit Mutation on \mLOTZ}

We recall the $\mLOTZFULL$ benchmark ($\mLOTZ$ for short) for many objectives. In $\mLOTZ(x)$, we divide $x$ in $m/2$ many blocks $x^j$ of equal length such that $x=(x^1, \ldots , x^{m/2})$ and in each block we count the $\LO$-value and the $\TZ$-value. This can be formalized as follows.
\begin{definition}[\citet{Laumanns2004}]
\label{def:mLOTZ}
Let $m$ be divisible by $2$ and let the problem size be a multiple of $m/2$. Then the $m$-objective function \mLOTZ is defined by
$\mLOTZ: \{0,1\}^n \to \mathbb{N}_0^m$ as 
\[
\mLOTZ(x) = (f_1(x), f_2(x), \ldots ,f_m(x))
\]
with 
\[
f_k(x)=
\begin{cases}
    \sum_{i=1}^{2n/m} \prod_{j=1}^i x_{j+n(k-1)/m}, & \text{ if $k$ is odd,} \\
    \sum_{i=1}^{2n/m} \prod_{j=i}^{2n/m} (1-x_{j+n(k-2)/m}), & \text{ else,}
\end{cases}
\]
for all $x=(x_1, \ldots ,x_n) \in \{0,1\}^n$.
\end{definition}

A Pareto-optimal set of \mLOTZ is
\[
\{1^{i_1}0^{2n/m-i_1} \ldots 1^{i_{m/2}}0^{2n/m-i_{m/2}} \mid i_1, \ldots , i_{m/2} \in \{0, \ldots , 2n/m\}\}
\]
which we denote by $P$ which coincides with its set of Pareto-optimal search points. The cardinality of this set is $(2n/m+1)^{m/2}$. In~\cite{OprisNSGAIII} it is shown that the size of a set of mutually incomparable solutions of \mLOTZ is at most $(2n/m+1)^{m-1}$. In the next lemma it is shown that the size of this set can be indeed asymptotically larger than the size of the Pareto front for a wide range of values $m$. 

\begin{lemma}
\label{lem:fitnessvectors-non-dom-LOTZ}
    Let $S$ be a maximum cardinality set of mutually incomparable solutions for $f:=\mLOTZ$. Then $|S| = n+1$ if $m=2$ and 
    $$\frac{(2n/m + 1)^{m-1}}{4(m-2)^{m/2-1}} \leq \lvert{S}\rvert \leq (2n/m+1)^{m-1}$$
    if $m \geq 4$.
\end{lemma}

\begin{proof}
Let $V:=f(S)$. We only show $\lvert{V}\rvert \geq k^{m-1}/(4(m-2)^{m/2-1})$ where $k:=2n/m+1$ as the upper bound is Lemma~4.2 in~\cite{OprisNSGAIII}. For $m=2$ a set $S$ with $f(S)=\{(n,0),(1,n-1), \ldots , (0,n)\}$ and cardinality $n+1$ is a set of mutually incomparable solutions. Suppose that $m \geq 4$. We construct a set $V' \subset \mathbb{N}_0^m$ with $\lvert{V'}\rvert \geq k^{m-1}/(4(m-2)^{m/2-1})$ such that there is a set $S'$ of mutually incomparable solutions with $f(S') = V'$ (which implies $\lvert{S'}\rvert = \lvert{V'}\rvert$ since $f(x) \neq f(y)$ for two distinct $x,y \in S'$, and $\lvert{S}\rvert \geq \lvert{S'}\rvert$ as $S$ has maximum possible cardinality). At first define for $w:=(w_1, \ldots , w_{m/2}) \in \{0, \ldots , k\}^{m/2}$
\[
M_w:=\{v \in \{0, \ldots, k\}^m \mid v_{2i-1}+v_{2i} = w_i \text{ for }i \in \{1, \ldots ,m/2\}\}.
\]
Then two search points $x,y$ with $f(x)=u$, $f(y)=v$ for $u,v \in M_w$ with $u \neq v$ are incomparable: Fix $i \in \{1, \ldots, m\}$ with $u_i \neq v_i$. If $u_i < v_i$, then $u_{i-1} > v_{i-1}$ if $i$ is even and $u_{i+1} > v_{i+1}$ if $i$ is odd. If $u_i > v_i$, then $u_{i-1} < v_{i-1}$ if $i$ is even and $u_{i+1} < v_{i+1}$ if $i$ is odd. 

For $w \in \{0, \ldots, k\}^{m/2}$ we have that $\lvert{M_w}\rvert = \prod_{i=1}^{m/2} (w_i+1)$ (since $v_{2i-1}+v_{2i}=w_i$ is possible for $i \in \{1, \ldots , m\}$ if and only if $(v_{2i-1},v_{2i}) \in \{(w_i,0), (w_i-1,1), \ldots, (1,w_i-1),(0,w_i)\}$). 
Further $M_w \cap M_{w'} = \emptyset$ for $w \neq w'$ as the sum $v_{2i-1}+v_{2i}$ is uniquely determined for a vector $v \in \mathbb{R}^m$. Now consider for $r:=\lceil{(m-3)k/(m-2)}\rceil$
\begin{align*}
W:=&\Big\{w \in \{0, \ldots , k\}^{m/2} \mid w_i \in [r,k] \text{ for } i \in \{1,\ldots ,m/2-1\} \text{ and } w_{m/2}=k - \sum_{i=1}^{m/2-1} (w_i-r) \Big\}.
\end{align*}
Then we have that
\begin{align*}
w_1 + \ldots + w_{m/2} &= k+\sum_{i=1}^{m/2-1} r = k+(m/2-1)r
\end{align*}
implying $v_1+ \ldots + v_m = k+(m/2-1)r$ for every $v \in M_w$ where $w \in W$. 
Consequently, two search points $x,y$ with $f(x)=u$ and $f(y)=v$ for $u \in M_{w_1}$ and $v \in M_{w_2}$ with distinct $w_1,w_2 \in W$ are incomparable: 
If there is a dominance relation between $x$ and $y$ we have that $u_1+ \ldots + u_m > v_1 + \ldots + v_m$ or $u_1+ \ldots + u_m < v_1+ \ldots + v_m$, but these both sums are $k+(m/2-1)r$. There is also no weak dominace relation between $x$ and $y$ since $u \neq v$ (because $w_1$ and $w_2$ are distinct). Thus for
\[
V':=\{v \in \{0, \ldots, k\}^m \mid v \in M_w \text{ for a } w \in W\}
\]
there is a set $S'$ of mutually incomparable solutions with $f(S')=V'$. We show $\vert{V'}\vert\geq (2n/m+1)^{m-1}/(4(m-2)^{m/2-1})$ and obtain the result. Since $w_i$ is bounded from below by $r$ for $i \in \{1, \ldots , m/2-1\}$ if $w \in W$, we obtain
\begin{align*}
\lvert{V'}\rvert = \sum_{w \in W} \lvert{M_w}\rvert &= \sum_{w \in W}\prod_{i=1}^{m/2} (w_i+1)\\
&\geq \sum_{w \in W}(r+1)^{m/2-1}(w_{m/2}+1) \\
    &= (r+1)^{m/2-1}\sum_{w \in W}(w_{m/2}+1).
     \end{align*}
     Since we have $w_{m/2}=k-\sum_{i=1}^{m/2-1}(w_i-r)$ and $w_i$ has range in $r, \ldots ,k$ if $w \in W$, we obtain
    \begin{align*}
     \sum_{w \in W}(w_{m/2}+1) &= \sum_{w_1=r}^{k} \ldots \sum_{w_{m/2-1}=r}^{k} \left(k+1-\sum_{i=1}^{m/2-1} (w_i-r) \right)\\
    &= \sum_{w_1=0}^{k-r} \ldots \sum_{w_{m/2-1}=0}^{k-r} \left(k+1-\sum_{i=1}^{m/2-1} w_i \right)
     \end{align*}
     and for every $j \in \{1,\ldots , m/2-1\}$
     \begin{align*}
     \sum_{w_1=0}^{k-r} \ldots \sum_{w_{m/2-1}=0}^{k-r} w_j &= (k-r+1)^{m/2-2} \cdot \sum_{w_j=0}^{k-r}w_j\\
     &= \frac{k-r}{2}(k-r+1)^{m/2-1} =:q
     \end{align*}
     where the latter equality is due to the Gaussian sum. Since $r=\lceil{(m-3)k/(m-2)}\rceil$, we see $k - r \leq k-(m-3)k/(m-2) = k/(m-2)$
     and we obtain
     \begin{align*}
     \sum_{w \in W}(&w_{m/2}+1) = (k+1) \cdot (k-r+1)^{m/2-1} - (m/2-1)q \\ &= (k-r+1)^{m/2-1}\left(k+1 - \frac{(m/2-1)(k-r)}{2}\right) \\
     &\geq (k-r+1)^{m/2-1} \left(k+1 - \frac{(m/2-1) k}{4(m/2-1)}\right) \\
     &\geq (k-r+1)^{m/2-1}(k+1 - k/4) \\
     &\geq k/2 \cdot (k-r+1)^{m/2-1}
     \end{align*}
     and consequently due to $\lceil{x}\rceil \geq x$ and $-x \leq -\lceil{x}\rceil+1$
     \begin{align*}
     \lvert{V'}\rvert & \geq \frac{k}{2} \cdot (r+1)^{m/2-1} \cdot (k-r+1)^{m/2-1}\\
     &\geq \frac{k}{2} \cdot \left(\frac{m-3}{m-2} \cdot k + 1\right)^{m/2-1} \cdot \left(k-\frac{m-3}{m-2} \cdot k\right)^{m/2-1}\\
     &\geq \frac{k}{2} \cdot \left(\frac{m-3}{m-2} \cdot k\right)^{m/2-1} \cdot \left(\frac{k}{m-2}\right)^{m/2-1}\\
     &= \frac{k^{m-1}}{2} \left(1-\frac{1}{m-2}\right)^{(m-2)/2}\left(\frac{1}{m-2}\right)^{m/2-1}\\
     &\geq \frac{k^{m-1}}{4(m-2)^{m/2-1}} = \frac{(2n/m+1)^{m-1}}{4(m-2)^{m/2-1}}
     \end{align*}
    where the last inequality holds, as $(1-1/\ell)^{\ell/2} \geq 1/2$ for every $\ell \geq 2$ and $m \geq 4$.
\end{proof}


In this section we analyze \PAES with one-bit mutation on \mLOTZ and establish tight bounds on the expected runtime to fill the archive with the full Pareto set $P$. A key part of the analysis involves proving that the potential function $W_t:=\sum_{i=1}^{m/2} (\LO(s^i) + \TZ(s^i))$, where $s$ denotes the current solution in iteration $t$, is non-decreasing. Observe that $0 \leq W_t \leq n$, and $s$ is Pareto-optimal if and only if $W_t=n$. Consequently, once a Pareto-optimal solution is found, it remains Pareto-optimal in all subsequent iterations. This monotonicity of the potential function $W_t$ also enables the derivation of an expected runtime bound for reaching the Pareto front.

\begin{lemma}
    \label{lem:approach-Pareto-current-solution-plus-pareto}
    Consider \PAES with one-bit mutation and any archiver optimizing \mLOTZ for any $m \leq n$ divisible by $2$ and archive size at least one. Let $W_t:=\sum_{i=1}^{m/2} (\LO(s^i) + \TZ(s^i))$. Then $W_{t+1} \geq W_t$. Furthermore, the expected number of iterations until the current solution is Pareto-optimal is $O(n^2)$. 
\end{lemma}

\begin{proof}
    The current solution $s$ either remains unchanged if the candidate solution $c$ is rejected or it is updated if the candidate solution is accepted. Suppose we flip bit $i \in [2n/m]$ in block $\ell$. If $i \in \{1, \ldots , \LO(s^\ell)\}$, then $\LO(c^\ell) < \LO(s^\ell)$ and all $\LO$- and $\TZ$-values of all other blocks remain unchanged. Similarly, if $i \in \{2n/m-\TZ(s^\ell)+1, \ldots , 2n/m\}$, then $\TZ(c^\ell) < \TZ(s^\ell)$, and again, all $\LO$- and $\TZ$-values of all other blocks remain unchanged. In both cases, $c$ is dominated by $s$ and $c$ is therefore discarded which implies that $s$ remains unchanged yielding $W_{t+1}=W_t$. Further, suppose that $i \in \{\LO(s^\ell)+2, \ldots , 2n/m-\TZ(s^\ell)-1\}$. Then all $\LO$- and $\TZ$-values remain unchanged and hence $W_{t+1}=W_t$ no matter if $s$ is updated. If $i = \LO(s^\ell)+1$ a zero is flipped to one and if $i=2n/m-\TZ(s^\ell)$ a one is flipped to zero. Hence, in the first case, $\LO(c^{\ell}) > \LO(s^\ell)$ while $\TZ(s^\ell)$ decreases only by at most one. In the second case $\TZ(s^\ell)>\TZ(c^\ell)$ while $\LO(s^\ell)$ decreases only by at most one. Hence, the current solution is updated yielding $W_{t+1} \geq W_t$, particularly $W_{t+1} > W_t$ if $s$ is not Pareto-optimal. All these cases together show $W_{t+1} \geq W_t$. Flipping one specific bit requires $O(n)$ iteration in expectation. Repeating this at most $n$ times (as long as $W_t < n$), derives the desired runtime bound concluding the proof.
\end{proof}




We now analyze how the current solution explores the Pareto front. When the current solution generates a new Pareto-optimal candidate $c$, it is accepted as the new current solution and added to the archive $A_t$ if $L$ is sufficiently high. If there already exists a solution $y \in A$ with $f(y)=f(c)$, then $y$ is replaced by $c$ in the archive. In both scenarios, the candidate $c$ becomes the current solution for the next iteration. This mechanism enables the current solution to perform a random walk on the Pareto front $P$. We will demonstrate that this walk can be modeled as a specific type of random walk on an \emph{$m$-dimensional grid graph}. 

To this end, we need some more notation. For an undirected connected graph $G$, a \emph{simple random walk} on $G=(V,E)$ is a random process in which, at each step, an edge is chosen uniformly at random from the set of edges incident to the current node, and the walk transitions to the next node along that edge. We denote by $T_{G,v}$ the time it takes for the simple random walk on $G$ to visit all nodes, starting at $v$. The \emph{cover time $\E(T_{G,v})$ with respect to $v \in V$} is the expected time it takes for the simple random walk on $G$ to visit all nodes, starting at $v$. 

For two undirected graphs $G:=(V_1,E_1)$ and $H:=(V_2,E_2)$, the \emph{product graph} $G \times H$ is the graph with node set $V_1 \times V_2$ and edge set $E$ where $\{(u_1,u_2),(v_1,v_2)\} \in E$ if and only if $u_1=v_1$ and $(u_2,v_2) \in E_2$ or $u_2 = v_2$ and $(u_1,v_1) \in E_1$. We write $G^2:=G \times G$ and inductively $G^k:=G \times G^{k-1}$ for $k \geq 3$. For example, if $G$ is a \emph{path graph}, that is its vertices can be labeled as $v_1, \ldots , v_n$ with edges $\{v_i,v_{i+1}\}$ for $i \in [n-1]$, then the product graph $G^k$ is the \emph{grid graph} $(V^k,E^k)$ of dimension $k$ and length $n$. Hereby, two $u,w \in V^k$ are adjacent if and only if there are $j,\ell \in [k]$ such that $w_j=v_\ell$ and $u_j \in \{v_{\ell-1},v_{\ell+1}\}$, and $u_i=w_i$ for all $i \in \{1,\ldots , n\} \setminus \{j\}$. We need the following result from~\citet{JONASSON2000181} about the cover time of the product graph $G^k$. Surprisingly, this depends only asymptotically on the \emph{average degree} $d_G:=2|E|/n$ of the underlying graph $G$ and the number of nodes in $G^k$, and not even on the choice of the starting node $v \in G^k$.

\begin{theorem}[\cite{JONASSON2000181}, Theorem~1.2]
Let $G=(V,E)$ be a connected graph with $n$ vertices and average degree $d_G$. Then the cover time $\E(T_{G,v})$ for a simple random walk on $G^k$ is $\Theta(d_G n^2 \log^2(n))$ if $k=2$ and $\Theta(d_G n^k \log(n^k))$ otherwise for every $v \in G^k$.  
\end{theorem}

For the remainder of this section, we mention the cover time without referring to the starting node, by convention. If $G$ is the path graph as mentioned above, then $d_G=2(n-1)/n \in \Theta(1)$ and hence, the cover time on $G^k$ is $\Theta(n^2 \log^2(n))$ if $k=2$ and $\Theta(n^k \log(n^k))$ for $k>2$. Now we are ready to prove the desired runtime bound of \PAES with 1-bit mutation on \mLOTZ.

\begin{theorem}
\label{thm:spreading-Pareto-front-Local}
Consider \PAES with one-bit mutation, an arbitrary archiver and archive size $L \geq (2n/m+1)^{m-1}$, or with the AGA and archive size $L \geq (2n/m+1)^{m/2}$, optimizing \mLOTZ for any $m \leq n$ where $m$ is divisible by $2$. Then the expected number of iterations until $A_t = P$ is $\Theta(n^3)$ if $m=2$, $\Theta(n^3\log^2(n))$ if $m = 4$ and $\Theta(n \cdot (2n/m)^{m/2}(\log(n/m)))$ if $m>4$.
\end{theorem}

\begin{proof}
By Lemma~\ref{lem:approach-Pareto-current-solution-plus-pareto} the first Pareto-optimal search point is created in expected $O(n^2)$ iterations. From this time on, the current solution always maintains Pareto-optimality by Lemma~\ref{lem:approach-Pareto-current-solution-plus-pareto}. At first we assume that the archiver is arbitrary and hence, $L \geq (2n/m+1)^{m-1}$. The candidate solution $c$ is always added to $A_t$ becoming the new current solution $s$ while possibly removing another point from $A_t$: If $A_t$ is not full or there is already an $x \in A_t$ with the same fitness as $c$, then it is obvious. Otherwise, $c$ dominates another $x \in A$ (since $L \geq |S|$ by Lemma~\ref{lem:fitnessvectors-non-dom-LOTZ} where $S$ is a maximum set of incomparable solutions) and $x$ is removed from $A_t$. Hence, we have to compute the expected time until $s$ attained all Pareto-optimal search points. To this end, we define the following graph $G=(V,E)$ and model the trajectory of $s$ by a random walk on $G$. Let $V:=\{(v_1, \ldots , v_{m/2}) \mid v_i \in \{0, \ldots , 2n/m\} \text{ for every } i \in [m/2]\}$ and $E:=\{\{v,w\} \mid v,w \in V, |v-w|=1\}$ where $|.|$ denotes the Euclidean distance in the objective space. Then $G = \tilde{G}^{m/2}$ where $\tilde{G}$ is the path graph with $2n/m+1$ nodes. Then the current solution $s=(s^1, \ldots , s^{m/2})$ visits $v:=(v_1, \ldots ,v_{m/2})$ if it is Pareto-optimal and for each $i \in [m/2]$ we have that $\LO(s^i) = v_i$. The transition probability between two adjacent nodes in the graph $\tilde{G}^{m/2}$ is $1/n$, as for a specific block $i$ the $\LO(s^i)$ increases or decreases by one with probability $1/n$, while the $\LO(s^j)$-values for $j \neq i$ remain unchanged. The probability that the node $v$ is not left is $1-|\text{adj}(v)|/n$ where $\text{adj}(v)$ is the number of adjacent nodes to $v$. 
Let $t_0$ be the first iteration when $s$ is Pareto-optimal. In iteration $t_0$, the random walk begins at $v \in V$, where $\LO(s^i)=v_i$ for each $i \in [m/2]$. A mutation of $s$ into $c$ corresponds to a transition from the node $v=(\LO(s^1), \ldots , \LO(s^{m/2}))$ to 
$w=(\LO(c^1), \ldots , \LO(c^{m/2}))$ if $c$ is also Pareto-optimal. 
The time (from iteration $t_0$) to obtain $A=P$ corresponds to the time $T$ of the random walk on $G=\tilde{G}^{m/2}$ described above to visit all nodes, starting at $v$. Note that the random walk is not simple, as for every node $v$ there is a certain probability to stay at $v$. We asymptotically determine $\E(T)$. Let $M \leq T$ be the number of moves the walk makes on $\tilde{G}^{m/2}$ to visit all nodes (i.e. idle steps where nodes are not left are ignored in $M$). Then $M$ is also the number of steps a \emph{simple random walk} has to make on the grid graph $\tilde{G}^{m/2}$ to visit all nodes. Denote by $X$ the waiting time to leave a node with the most neighbors and by $Y$ the waiting time to leave a node with the least ones. Then $\E(X) = n/(2m)$ and $\E(Y) = n/m$ as $\text{adj}(v) \in \{m, \ldots , 2m\}$. Note that $M$ can be regarded as independent to $X$ and $Y$ and therefore $\E(X) \cdot \E(M) \leq E(T) \leq \E(Y) \cdot \E(M)$ which implies (since the number of nodes in $\tilde{G}^{m/2}$ is $(2n/m+1)^{m/2}$) 
\begin{align*}
\E(T) &= \Theta(n/m \cdot (2n/m+1)^{m/2}\log((2n/m+1)^{m/2})) \\
&= \Theta(n \cdot (2n/m)^{m/2}(\log(n/m)))
\end{align*}
if $m \geq 6$ and
\begin{align*}
\E(T) &= \Theta(n/m \cdot (2n/m+1)^{m/2}\log^2(2n/m+1)) \\
&= \Theta(n^3 \log^2(n))
\end{align*}
if $m = 4$. If $m=2$ we have that $\tilde{G}^{m/2}$ is a path graph and therefore $\E(X)=\E(Y)= \Theta(1/n)$ and $\E(T) = \Theta(n^3)$  since $\E(M)=\Theta(n^2)$ (by Theorem~1 in \cite{IKEDA200994}). This concludes the proof for arbitrary archivers with size $L$ of at least $(2n/m+1)^{m-1}$. In the case of AGA, any solution that is incomparable to all solutions in $A_t$ is always accepted as the new candidate solution. Consequently, once the Pareto front has been reached, the same random walk performed by the current solution as above eliminates any non Pareto-optimal search points from the archive by visiting all Pareto-optimal solutions. In a second phase of visiting all nodes again, if necessary, it fills the archive with Pareto-optimal solutions. This shows that the stated expected tight runtime bound also holds in this case since every Pareto-optimal search point needs to be visited by the current solution at most twice, but at least once.
\end{proof}

For other MOEAS like (G)SEMO only an upper bound of $O(n^{m+1})$ is shown on \mLOTZ (see~\cite{Laumanns2004,DoerrNearTight}). Theorem~\ref{thm:spreading-Pareto-front-Local} shows that we beat this upper runtime bound by a factor of $n^2/\log^2(n)$ if $m=4$ and $(nm/2)^{m/2}/(\log(n/m))$ if $m>4$. To the best of our knowledge, these are the first tight runtime bounds of an MOEA in the many-objective setting.

For the remainder of this section we show that \PAES with one-bit mutation can efficiently create a well-distributed archive on the Pareto front of \mLOTZ when using AGA, HVA and MGA, assuming that the archive size $L$ is smaller than the size of the Pareto front. In the case of HVA and MGA, we restrict our analysis to the bi-objective setting.

\textbf{AGA:} \PAES with AGA efficiently distributes all solutions in $A_t$ very evenly
across the cells that contain Pareto-optimal search points, stated as follows.

\begin{lemma}
\label{lem:spreading-Pareto-front-1-Bit-Mutation-AGA}
Consider \PAES with one-bit mutation, AGA and archive size $L < (2n/m+1)^{m/2}$ optimizing \mLOTZ for any $m \leq n$ where $m$ is divisible by $2$. Then after expected $O(n^3)$ iterations if $m=2$, $O(n^3\log^2(n))$ iterations if $m = 4$ and $O(n \cdot (2n/m)^{m/2}(\log(n/m)))$ iterations if $m>4$ we obtain the following distribution of solutions from $A_t$ across all cells that contain Pareto-optimal search points. For any cell containing $k$ Pareto-optimal solutions from $A_t$, the following holds: 
\begin{itemize}
    \item[(1)] Every other cell containing at least $k$ distinct Pareto-optimal solutions from $\{0,1\}^n$ contains at least $k-1$ Pareto-optimal solutions from $A_t$.
    \item[(2)] Every cell containing $\ell \leq k-1$ Pareto-optimal solutions from $\{0,1\}^n$ contains $\ell$ Pareto-optimal solutions from $A_t$.
\end{itemize}
\end{lemma}

\begin{proof}
With Lemma~\ref{lem:approach-Pareto-current-solution-plus-pareto} we can assume that $A_t \subset P$. The current solution performs the same random walk as in the proof of Theorem~\ref{thm:spreading-Pareto-front-Local}. Hence after $\Theta(n^3)$ iterations if $m=2$, $\Theta(n^3\log^2(n))$ iterations if $m = 4$ and $\Theta(n \cdot (2n/m)^{m/2}(\log(n/m)))$ iterations if $m>4$ in expectation, the current solution has visited all Pareto-optimal fitness vectors. From this time on, $A_t$ consists only of Pareto-optimal solutions since dominated solutions are removed from $A_t$ (see Algorithm~\ref{alg:PAES25}). Finally, to fill and spread $A_t$ on the Pareto front, we repeat the process a second time to visit all Pareto-optimal fitness vectors again. Suppose that after the final iteration in this second phase there exists a cell $C$ containing $k$ Pareto-optimal solutions from $A_t$, and another cell $C'$ that contains at most $\ell \leq k-2$ Pareto-optimal solutions from $A_t$, but at least $\ell+1$ Pareto-optimal solutions from $P$. Then, during an earlier iteration, at least one solution from $P \cap C'$ must have been removed from $A_t$ when cell $C'$ had exactly $\ell + 1$ Pareto-optimal solutions from $A_t$ (since all Pareto-optimal solutions have been visited after the second phase). At this time $t$, $A_t$ was full and every other cell (including $C$) contained at most $\ell +1<k$ solutions from $A_t$. But AGA always removes Pareto-optimal solutions from $A_t$ contained in cells with the largest number of Pareto-optimal solutions from $A_t$. Hence, for any iteration $t' \geq t$, once a cell is filled with an $(\ell+2)$-th Pareto optimal search point from $A_t$, one of these $(\ell+2)$ points is removed. Consequently, each cell contains at most $\ell + 1 < k$ solutions from $A_{t'}$ after iteration $t'$ which is a contradiction proving properties (1) and (2).      
\end{proof}
\textbf{HVA:} We show that \PAES with HVA and reference point $h=(-1,-1)$ efficiently distributes the archive $A_t$ across the Pareto front, resulting in a high hypervolume. We restrict our analysis to the case $m=2$, as for $m > 2$, there may be non-dominated points in $A_t$ which are not Pareto-optimal even if a Pareto-optimal point is already found. This complicates the analysis since non Pareto-optimal points in $A_t$ may have a greater hypervolume contribution than newly generated Pareto-optimal points from the random walk, potentially causing the algorithm to reject new Pareto-optimal solutions. At first we establish some basic properties and compute the hypervolume of $A_t$ under the assumption that all points in $A_t$ are Pareto-optimal and there are no \emph{holes}, meaning that for each $x,y \in A_t$ and $z \in P$ with $\LO(x) < \LO(z) < \LO(y)$, we also have $z \in A_t$. We also derive a formula for the case $A_t=P$, which is the maximum possible hypervolume for an archive $A_t$. 

\begin{lemma}
\label{lem:hypervolume-basic-properties}
  Let $A_t \subset P$ be without a hole. Let $a:=\min\{\LO(x) \mid x \in A_t\}$ and $b:=\max\{\LO(x) \mid x \in A_t\}$. Then we have 
  \begin{align*}
    \hv(A_t) = (n+1)(b+1) - \frac{a(a+1)}{2}  - \frac{b(b+1)}{2}
  \end{align*}
  with respect to the reference point $h=(-1,-1)$. If $d:=b-a$ is fixed then the minimum hypervolume is attained when $a=0$ or $b=n$ which is $\hv(A_t) = (d+1)(n+1-d/2)$.
\end{lemma}

\begin{proof}
Note that $f(x)=(a,n-a)$ for a Pareto-optimum $x$ with $\LO(x)=a$ and hence, $x$ dominates all $y$ with $-1 \leq f_1(y) \leq a$ and $-1 \leq f_2(y) \leq n-a$ which contribute $(a+1)(n-a+1)$ to $\hv(A_t)$. For $a < i \leq b$ all search points $z$ with $i-1 < f_1(z) \leq i$ and $-1 \leq f_2(z) \leq n-i$ are dominated by a search point $x \in A_t$ with $\LO(x)=i>a$ (since $\TZ(x) = n-i$), but not by a search point $y \in A_t$ with $\LO(y)<i$. This contributes $n-i+1$ to $\hv(A_t)$. Therefore, we obtain
\begin{align*}
& \hv(A_t) = (a+1)(n-a+1) + \sum_{i=a+1}^b (n-i + 1)\\
& = (a+1)(n-a+1) + \sum_{i=0}^{b} (n-i + 1)  - \sum_{i=0}^a (n-i + 1)\\
& = (a+1)(n-a+1) + (n+1)(b-a) + \sum_{i=1}^{a} i  - \sum_{i=1}^{b} i\\
& = (n+1)b-a(a+1)+n + 1 + \frac{a(a+1)}{2}  - \frac{b(b+1)}{2}
\end{align*}
proving the first formula. Now suppose that $d=b-a$ is fixed. If $a$ increases by one then $b$ also increases by one. Therefore, $\hv(A_t)$ increases by
\begin{align*}
&(n+1)(b+2) - \frac{(a+1)(a+2)}{2}  - \frac{(b+1)(b+2)}{2} \\
&- (n+1)(b+1) + \frac{a(a+1)}{2}  + \frac{b(b+1)}{2} \\ 
&= n+1-(a+1)-(b+1) = n-a-b-1.
\end{align*}
Hence, if $b \geq n-a$ the hypervolume is smallest if $b=n$ and $a=n-d$ which is 
\begin{align*}
& (n+1)^2 - \frac{(n-d)(n-d+1)}{2} - \frac{n(n+1)}{2} \\ 
&= (n+1)^2 - \frac{(n-d)(n+1)}{2} + \frac{d(n+1-d-1)}{2} - \frac{n(n+1)}{2} \\ 
&= (n+1)\left(n+1-\frac{n-d}{2}+\frac{d}{2}-\frac{n}{2}\right) - \frac{d(d+1)}{2} \\ 
&= (n+1)(d+1)-\frac{d(d+1)}{2} = (d+1)\left(n+1-\frac{d}{2}\right).
\end{align*}
Similarly, if $a$ decreases by one then $b$ decreases also by one and the hypervolume increases by $a+b-n-1$. If $b < n-a$ then the hypervolume is smallest if $a=0$ and $b=d$ which is
\begin{align*}
(n+1)(d+1)-\frac{d(d+1)}{2} = (d+1)\left(n+1-\frac{d}{2}\right)
\end{align*}
proving the result.
\end{proof}

We now show that \PAES with HVA and one-bit mutation on \LOTZ effectively distributes the archive across the Pareto front. To this end, we set $a_t:=\min\{\LO(x) \mid x \in A_t\}$, $b_t:=\max\{\LO(x) \mid x \in A_t\}$ and $d_t:=b_t-a_t$. Call a Pareto-optimal $z$ \emph{hole} if $a_t<\LO(z)<b_t$ and $z \notin A_t$.

\begin{lemma}
\label{lem:spreading-Pareto-front-1-Bit-Mutation-HVA}
Consider \PAES with one-bit mutation, HVA with $h=(-1,-1)$, and archive size $L < n+1$ optimizing \LOTZ. Then after expected $O(n^3)$ iterations the following holds.
\begin{itemize}
    \item[(1)] For any given archive size $L$ with $L+\lceil{L/2}\rceil \leq n+2$, all solutions in $A_t$ are distributed such that $d_t = L+\lceil{L/2}\rceil-2$. Furthermore, there are $\lceil{L/2}\rceil-1$ holes and two of them are not neighbored.
    \item[(2)] For any given archive size $L$ with $L+\lceil{L/2}\rceil \leq n+2$, a hypervolume of at least
    $$(L+\lceil{L/2}\rceil-1) \cdot \left(n+1-\frac{L+\lceil{L/2}\rceil-2}{2}\right) -\lceil{L/2}\rceil+1$$
    is attained. If $L=cn$ for some $0<c\leq 1$ this corresponds to at least a fraction of $3c(1-3c/4)(1 - o(1))$ of the total hypervolume $\hv(P)$. If instead $L+\lceil{L/2}\rceil > n+2$, then a hypervolume of $(n+2)(n+1)/2 - \max\{n+1-L,0\}$ is achieved which is also the maximum possible in this case. 
\end{itemize}
\end{lemma}

\begin{proof}
With Lemma~\ref{lem:approach-Pareto-current-solution-plus-pareto} we can assume that $A_t \subset P$. For $L=1$ the lemma holds since $\hv(A_t) \geq n+1$. Suppose that $L \geq 2$. As long as solutions $x,y,z \in A_t$ with $\LO(x)=\ell$, $\LO(y)=\ell+1$, and $\LO(z)=\ell+2$ exist, the solution $y$ contributes a hypervolume of one which is smallest possible. Therefore, if the candidate solution $c$ is incomparable to all solutions in $A_t$, its contribution cannot be smaller than that of $y$, and it is added to $A_t$. As long as there are fewer than $\lfloor{(d_t+1)/3}\rfloor$ holes, such $x,y,z$ can be found, and the current solution conducts the same random walk as in the proof of Theorem~\ref{thm:spreading-Pareto-front-Local} and a new created Pareto-optimal solution is always added to $A_t$ (possibly replacing another solution with the same fitness). Hence, after $O(n^3)$ iterations in expectation, the current solution will have visited all Pareto-optimal points if it does not get stuck. The latter can only happen if the archive is full and there are at least $\lfloor{(d_t+1)/3}\rfloor$ holes. The latter implies $L \leq (d_t+1) - \lfloor{(d_t+1)/3}\rfloor$ yielding
\begin{align*}
L +\lceil{L/2}\rceil-2 &\leq \frac{3L}{2}-\frac{3}{2} \leq \frac{3(d_t+1)}{2} - \frac{3\lfloor{(d_t+1)/3}\rfloor}{2}-\frac{3}{2} \\
& \leq \frac{3(d_t+1)}{2} - \frac{3((d_t+1)/3-2/3)}{2}- \frac{3}{2} \\
&= \frac{3(d_t+1)-(d_t+1)+2}{2} - \frac{3}{2} = d_t+\frac{1}{2}
\end{align*}
and therefore $L +\lceil{L/2}\rceil-2 \leq d_t$. 

Suppose that $L+\lceil{L/2}\rceil \leq n+2$. Then in $O(n^3)$ iterations in expectation an iteration $t^*$ is reached when $A_{t^*}$ is full and $L +\lceil{L/2}\rceil-2 = d_{t^*} = b_{t^*}-a_{t^*}$ since $L+\lceil{L/2}\rceil-2 > d_t$ for all iterations $t$ before $t^*$. The number of search points $x$ in $P$ with $a_t \leq \LO(x) \leq b_t$ is $d_t+1$ and hence, there are $d_t+1-L=\lceil{L/2}\rceil-1$ holes. Note that in all iterations before $t^*$ only solutions from $A_{t^*}$ are removed which have a hypervolume contribution of at most one to $\hv(A_{t^*})$ (since there is $\ell \in \{0, \ldots , n-2\}$ and $x,y,z$ with $\LO(x)=\ell$, $\LO(y)=\ell+1$ and $\LO(z)=\ell+2$) and hence, two holes in $A_{t^*}$ are never neighbored. This proves property~(1). Further, by Lemma~\ref{lem:hypervolume-basic-properties} we see that 
\[
\hv(A_t) \geq (L +\lceil{L/2}\rceil-1)\left(n+1-\frac{L +\lceil{L/2}\rceil-2}{2}\right) - \lceil{L/2}\rceil+1
\]
at iteration $t:=t^*$ since plugging a hole with an additional search point would increase $\hv(A_t)$ by one. If $L=cn$ for some $0<c<1$ this is at least a fraction of
\begin{align*}
&\frac{(cn+\lceil{cn/2}\rceil-1) \cdot \left(n+1-\frac{cn+\lceil{cn/2}\rceil-2}{2}\right) -\lceil{L/2}\rceil+1}{{(n+2)(n+1)/2}}\\
\geq & \frac{(3cn/2-1) \cdot (n+3/2-3cn/4) -\lceil{L/2}\rceil+1}{{(n+2)(n+1)/2}}\\
\geq &\frac{3c/2 \cdot (1-3c/4)n^2(1-o(1))}{(n+2)(n+1)/2} = 3c(1-3c/4)(1 - o(1)).
\end{align*}
Suppose $L+\lceil{L/2}\rceil > n+2$. This implies $L + L/2 + 0.5 \geq n+3$ and hence $L \geq 2(n+2.5)/3=2n/3+5/3$. The number of holes is at most $d+1-L \leq d+1-(2n/3+5/3) \leq d+1-(2d/3+5/3) = d/3-2/3 < \lfloor{(d+1)/3}\rfloor$. Hence, $c_t$ performs the same random walk as in the proof of Theorem~\ref{thm:spreading-Pareto-front-Local} where only search points with hypervolume contribution of one to $\hv(A_t)$ are removed from $A_t$. Hence, after $O(n^3)$ iterations in expectation we have $a=0$, $b=n$ and there are $k:=\max\{n+1-L,0\}$ holes. Plugging such a hole with an additional search point increases $\hv(A_t)$ by one, and hence, we see with Lemma~\ref{lem:hypervolume-basic-properties}
\begin{align*}
\hv(A_t) &= (n+1)(n+1) - \frac{n(n+1)}{2} - k\\
&= (n+1)\left(n+1-\frac{n}{2}\right) - k = \frac{(n+1)(n+2)}{2}-k
\end{align*}
proving the result.  
\end{proof}

For instance, Lemma~\ref{lem:spreading-Pareto-front-1-Bit-Mutation-HVA} shows that when $h=(-1,-1)$, a fraction of approximately $15/16$ of the total hypervolume of $P$ can be attained by $A_t$, if its size is $n/2$, since $3/2 \cdot (1-3/8)(1-o(1)) \approx 15/16$.

\textbf{MGA:} We also show that \PAES with MGA efficiently creates an archive $A_t$, in which any two points $x,y \in A_t$ have different box index vectors with respect to a certain level $b>0$, where the level $b$ depends on the archive size. The smaller the archive size, the larger the corresponding level $b$. A higher level $b$ leads to a finer distribution of solutions across the Pareto front. As for HVA, we restrict our analysis to the bi-objective case, since in higher dimensions, non Pareto-optimal vectors may dominate Pareto-optimal ones at level $b>0$. This can result in a newly discovered Pareto-optimal solution not being added to $A_t$ if the archive is full, thereby complicating the analysis. We start with some elementary properties.

\begin{lemma}
\label{lem:different-coarseness-properties}
Let $n=2^k \cdot \ell-1$ where $k \in \mathbb{N}_0$ and $\ell$ is odd. Then the following holds:
\begin{itemize}
\item[(1)] Two distinct solutions $x,y \in P$ either have identical or mutually incomparable box index vectors at coarseness level $j \in \{0, \ldots , k\}$.
\item[(2)] There are exactly $(\ell+1)/2$ incomparable box index vectors at coarseness level $k+1$ and exactly $2^{k-j}\ell$ incomparable box index vectors at coarseness level $j \in \{0, \ldots , k\}$.
\end{itemize}
\end{lemma}

\begin{proof}
(1): Let $x \in P$ be a Pareto-optimal solution. We find $a \in \{0, \ldots , 2^{k-j} \cdot \ell-1\}$ and $b \in \{0, \ldots , 2^j-1\}$ with 
  \begin{align*}
  \LOTZ(x)&=(2^j \cdot a + b,n-2^j \cdot a-b) \\
  &= (2^j \cdot a + b,2^k \ell-2^j \cdot a-(b+1)).
  \end{align*}
The corresponding box index vector at level $j$ is $(a,2^{k-j}\ell-a-1)$. These are identical for the same $a$, but incomparable for different $a$.

(2): The components of a box index vector at coarseness level $k+1$ have a range of $\{0, \ldots , (\ell-1)/2\}$  and the box index vectors of search points $x$ with $(\LO(x),\TZ(x)) \in \{(0,n), (2^{k+1},n-2^{k+1}), \ldots, (\ell-1)/2 \cdot 2^{k+1},n-(\ell-1)/2 \cdot 2^{k+1})$ at coarseness level $k+1$ are $\{(0,(\ell-1)/2),(1,(\ell-1)/2-1), \ldots , ((\ell-1)/2,0)\}$ and therefore, pairwise incomparable. The components of a box index vector at coarseness level $j \in \{0, \ldots , k\}$ have a range of $\{0, \ldots , 2^{k-j}\ell-1\}$ and the box index vectors of search points $x$ with $(\LO(x),\TZ(x)) \in \{(0,n), (2^j,n-2^j), \ldots, (2^k\ell - 2^j, n - (2^k\ell - 2^j))\}$ at coarseness level $j$ are $\{(0,2^{k-j} \cdot \ell-1),(1,2^{k-j} \cdot \ell-2), \ldots , (2^{k-j} \cdot \ell-1,0)\}$ and therefore, pairwise incomparable.
\end{proof}

\begin{lemma}
\label{lem:spreading-Pareto-front-1-Bit-Mutation-MGA}
Consider \PAES with one-bit mutation and archive size $L \leq n$ optimizing \LOTZ. Let $n=2^k \cdot \ell-1$ where $k \in \mathbb{N}_0$ and $\ell$ is odd. Then after expected $O(n^3)$ iterations the following holds. If $L \leq (\ell+1)/2$, all solutions in $A_t$ have mutually incomparable box index vectors at coarseness level $k+1$. If $(\ell+1)/2 < L \leq \ell$, all solutions in $A_t$ have mutually incomparable box index vectors at coarseness level $k$ and if $2^{j-1}\ell < L \leq 2^j \ell$ for $j \in [k]$ all solutions in $A_t$ have mutually incomparable box index vectors at coarseness level $k-j$.
\end{lemma}

\begin{proof}
With Lemma~\ref{lem:approach-Pareto-current-solution-plus-pareto} we can assume that $A_t \subset P$. Then, Lemma~\ref{lem:different-coarseness-properties}~(1) shows that the current solution $s$ is always added to $A_t$ if Algorithm~\ref{alg:PAES25} makes a decision at coarseness level $j$. It is also added if Algorithm~\ref{alg:PAES25} makes a decision at coarseness level $k+1$ and there is a solution $x \in A_t$ whose box index vector is weakly dominated by the box index vector of a solution from $A_t \cup \{c\}$ at coarseness level $k+1$. Here $c$ is the candidate solution created from $s$ by mutation. As long as one of these two happens, the current solution $s$ performs the same random walk as in the proof of Theorem~\ref{thm:spreading-Pareto-front-standard}. Suppose that $L \leq (\ell+1)/2$. Note that there are at least $L$ many search points with pairwise incomparable box index vectors at coarseness level $k+1$ by Lemma~\ref{lem:spreading-Pareto-front-1-Bit-Mutation-MGA}. These are found by the random walk after expected $O(n^3)$ iterations since it is sufficient that the current solution visits all Pareto-optimal points. The same happens if $(\ell+1)/2 < L \leq \ell$ with coarseness level $k$ instead of $k+1$ and if $2^{j-1} \ell < L \leq 2^j \ell$ with coarseness level $k-j$ instead of $k+1$ for $j \in [k-1]$. This concludes the proof.
\end{proof}

\section{Runtime Analysis of \PAES with Standard Bit Mutation on \LOTZ}

In this section, we investigate the performance of \PAES on the bi-objective \LOTZ benchmark using standard bitwise mutation. The experimental results reported in~\cite{Knowles2025} indicate that \PAES with standard bit mutation requires more time to discover a constant fraction of Pareto-front than \PAES with one-bit mutation. This slower progress arises because the current solution may not remain Pareto-optimal even after a point on the Pareto front has been found. For instance, a two-bit flip can transform a Pareto-optimal solution into a new, non-dominated one (with respect to the archive $A_t$) by simultaneously increasing the \LO-value and decreasing the \TZ-value by more than one. As a result, the search dynamics under standard bit mutation are more complex than under one-bit mutation, and the resulting random walk may involve non-Pareto-optimal solutions. Rather than analyzing this random walk in detail, we instead focus on the hypervolume $\hv(A_t)$ of the archive $A_t$. We show that $\hv(A_t)$ cannot decrease in any iteration of \PAES with respect to any reference point $h$ satisfying $h_i \leq 0$ for each $i \in [m]$. To maintain generality, we present our result for an arbitrary fitness function $f:\{0,1\}^n \to \mathbb{N}_0^m$. 

\begin{lemma}
\label{lem:hypervolume-not-decrease}
Let $f:\{0,1\}^n \to \mathbb{N}_0^m$ be a fitness function and $S$ be a maximum set of mutually incomparable solutions with respect to $f$. Let $t$ be any iteration of \PAES with current archive $A_t$ with size of at least $|S|$ on $f$. Then $\hv(A_{t+1}) \geq \hv(A_t)$ with respect to any reference point $h$ with $h_i \leq 0$ for all $i \in [m]$. Particularly, if the candidate solution $c$ dominates an $x \in A_t$, $\hv(A_{t+1}) > \hv(A_t)$. 
\end{lemma}

\begin{proof}
Note that a search point will only be removed from $A_t$ if a candidate solution $c$ is generated that is added to $A_t$. If $A_t$ is full and the candidate solution $c$ is added to $A_t$, $c$ weakly dominates an $x \in A_t$ (since $|A_t| \geq |S|$). If there is $x \in A_t$ with $f(x)=f(c)$ then $\hv(A_{t+1})=\hv(A_t)$. Suppose not and let $B \subset A_t$ be the set of search points dominated by $c$. Then $\hv(A_{t+1}) = \hv((A_t \setminus B) \cup \{c\}) = \hv(A_t \cup \{c\})$. Let $v:=f(c)$. Since any search point $z$ with $v_i-1/2 \leq f_i(z) \leq v_i$ for every $i \in [m]$ is weakly dominated by $c$, but not weakly dominated by any $y \in A_t$, we have that $\hv(A_t \cup \{c\}) > \hv(A_t)$. 
\end{proof}

Using the maximum possible hypervolume of an archive $A_t$ with respect to the reference point $h=(-1,-1)$ which occurs when $A_t=P$, we can estimate the expected number of iterations required for the current solution to reach the Pareto front when \PAES optimizes \LOTZ.   

\begin{lemma}
\label{lem:Reach-Pareto-First-Time}
Consider \PAES with standard bit mutation optimizing $m$-\LOTZ for $m = 2$. Then for every possible $A_t$ the current solution $s$ is Pareto-optimal in expected $O(n^3)$ iterations. 
\end{lemma}

\begin{proof}
  Suppose that $c \notin P$. By flipping one specific bit to increase the $\LO$ or $\TZ$ value of $s$ while the other bits remain unchanged (happening with probability at least $1/(en)$) one can create a solution $c$ which (strictly) dominates $s$. Since $A_t$ contains no solution dominating $s$, and dominance is transitive, $c$ is not weakly dominated by any solution in $A_t$, and is therefore added to $A_t$ which becomes the current solution for the next iteration. By Lemma~\ref{lem:hypervolume-not-decrease} the hypervolume $\hv(A_t)$ increases. By Lemma~\ref{lem:hypervolume-basic-properties} we obtain (for $a=0$ and $b=n$)
  \begin{align*}
  \hv(P) &= (n+1)(n+1)-\frac{n(n+1)}{2} = \frac{(n+2)(n+1)}{2} = O(n^2)
  \end{align*}
  and we see that $\hv(A_t)$ cannot decrease by Lemma~\ref{lem:hypervolume-not-decrease}. Hence, the expected number of iterations until the current solution is Pareto-optimal is $O(n^3)$.
\end{proof}

We will also show that, in $O(n^3)$ expected time, a new Pareto-optimal search point is added to the archive, even when some Pareto-optimum already exists in $A_t$. However, establishing this result requires a different proof technique than used in Lemma~\ref{lem:Reach-Pareto-First-Time}, as the current solution may be also Pareto-optimal. This insight is key to deriving the overall runtime of \PAES on \LOTZ as shown in the following theorem.

\begin{theorem}
\label{thm:spreading-Pareto-front-standard}
Consider \PAES with standard bit mutation and archive size at least $n+1$ optimizing \mLOTZ for $m = 2$. Then the expected number of iterations until some Pareto-optimal search point $x$ with $x \notin A_t$ is added to $A_t$ is $O(n^3)$. This implies that the expected number of iterations until $A_t=P$ is at most $O(n^4)$.
\end{theorem}

\begin{proof}
    By Lemma~\ref{lem:Reach-Pareto-First-Time} the expected number of iterations until $s \in P$ is $O(n^3)$. Now suppose that $s \in P$ and not the complete set $P$ of Pareto-optimal individuals is contained in $A_t$. Consider $b^{\text{left}}:=b^{\text{left}}(s):=\sup\{\ell \in \{0, \ldots , n-1\} \mid \ell < \LO(s) \text{ and there is no }x \in A_t \cap P \text{ with } \LO(x)=\ell\}$ and $b^{\text{right}}(s):=b^{\text{right}}:=\inf\{\ell \in \{1, \ldots , n\} \mid \LO(s) < \ell \text{ and there is no }x \in A_t \cap P \text{ with } \LO(x)=\ell\}$. Then $b^{\text{left}}$ denotes the largest $\LO$ value among the individuals from $P \setminus A_t$ that have a smaller $\LO$ value than $s$. If all such individuals with a smaller $\LO$ value than $s$ are already present in $A_t$, we have $b^{\text{left}} = -\infty$.
    
\begin{figure}
    \centering
    \includegraphics[scale=0.6]{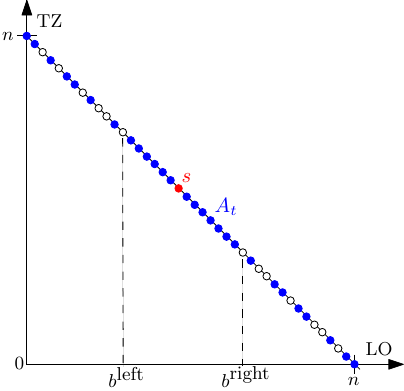}
    \caption{Illustration of $b^{\text{left}}$ and $b^{\text{right}}$ for the case $b^{\text{left}} \neq -\infty$ and $b^{\text{right}} \neq \infty$. The current solution $s$ quite in the middle is depicted in red and the other Pareto-optimal points from the archive $A_t$ are represented by the remaining filled points in blue. The white points with black outlines are the holes.}
    \label{fig:borders}
\end{figure}
    
    Similarly, $b^{\text{right}}$ denotes the smallest $\LO$ value among the individuals from $P \setminus A_t$ that have a larger $\LO$ value than $s$ (see also Figure~\ref{fig:borders}). If there is no such individual with a larger $\LO$ value than $s$, then $b^{\text{right}} = \infty$. As long as $b^{\text{left}}+1 \leq \LO(s) \leq b^{\text{right}}-1$, the current solution $s$ remains Pareto-optimal (since a solution $y$ with $b^{\text{left}}+1 \leq \LO(y) \leq b^{\text{right}}-1$ is dominated by the Pareto-optimal $z \in P_t$ with $\LO(z)=\LO(y)$) and the $\LO(s)$-value of the current solution $s$ performs a random walk on the discrete interval $[b^{\text{left}}+1,b^{\text{right}}-1] \cap \{0, \ldots , n\}$. 
    
    Since $s$ remains Pareto-optimal during this process, it is updated in each iteration when the candidate solution $c$ is also Pareto-optimal. Note also that $b^{\text{left}} \neq -\infty$ or $b^{\text{right}} \neq \infty$ since $A_t \neq P$.
    
    Denote an iteration as a \emph{failure} if one of the following two conditions is met. 
    \begin{itemize}
        \item[(1)]
        The value of $\LO(s)$ changes by at least four.
        \item[(2)]
        The value $\LO(s)$ exits $I:=[b^{\text{left}} + 1, b^{\text{right}} - 1]$, without resulting in a Pareto-optimal individual $x$ with $x \notin A_t$.
    \end{itemize}
    In the following we show that in $O(n^3)$ iterations in expectation a failure occurs or a Pareto-optimal search point $x$ with $x \notin A_t$ is added to $A_t$. To this end, it is sufficient to show that in expected $O(n^3)$ iterations the value $\LO(s)$ exits $I$ or a failure occurs. The aim is to apply Theorem~\ref{thm:unbiased-random-walk}. However, there may be a negative drift away from the borders of $I$ as the example $b^{\text{right}} = \infty$, $b^{\text{left}}=0$ and $s=10^{n-1}$ shows: Creating $c=0^n$ happens with the same probability as creating $c=1^20^{n-2}$ and we have a positive probability of creating $c=1^j 0^{n-j}$ for $j \geq 3$. Hence, this theorem is not straightforwardly applicable. However, this negative drift vanishes, if no failure occurs and $s \in [b^{\text{left}} +4,b^{\text{right}} - 4]$, allowing the theorem to apply for proving the following claim. 

    \begin{claim}
    \label{cl:exit-smaller-interval}
    After $O(n^3)$ iterations in expectation a failure occurs, or $\LO(s) \notin [b^{\text{left}} +4,b^{\text{right}} - 4]$ is satisfied for every possible $b^{\text{left}}<b^{\text{right}}$ with $b^{\text{left}} \neq -\infty$ or $b^{\text{right}} \neq \infty$, and for every possible Pareto-optimal starting point $s_0$ with $\LO(s_0) \in [b^{\text{left}} +4,b^{\text{right}} - 4]$.
    \end{claim}
    
    \begin{proofofclaim}
    We may assume that no failure occurs, as any failure would only accelerate the achievement of the goal outlined in this claim. Without loss of generality assume that $\LO(s) \in [b^{\text{left}} +4,b^{\text{right}} - 4]$ for all $0 \leq t \leq T$ where $T:=\inf\{t \geq 0 \mid \LO(s) \notin [b^{\text{left}} +4,b^{\text{right}} - 4]\}$. For $j \in \{-3, \ldots , 3\}$ denote by $p_j$ the probability that $\LO(s)$ changes by $j$ (which also means that $s$ stays Pareto-optimal). Then $p_{-j} = p_j$ and $p_j = \Theta(1/n^j)$ since it is required to flip $j$ specific bits. Let $X_t:=\LO(s)$ where $s$ denotes the current solution in iteration $t$. We consider three cases.

    \textbf{Case 1:} Suppose that $b^{\text{right}} \neq \infty$ and $b^{\text{left}} \neq -\infty$ and let $d:=(b^{\text{right}}+b^{\text{left}})/2$. Let $Y_t:=|X_t-d|$. Then we have
    \begin{align*}
    &\E(Y_{t+1}-Y_t \mid X_0, \ldots , X_t) = \E(|X_{t+1}-d|-|X_t-d| \mid X_0, \ldots , X_t).\\
    \intertext{If $X_t \geq d$ we obtain}
    &= \E(|X_{t+1}-d|-X_t+d \mid X_0, \ldots , X_t)\\
    &\geq \E(X_{t+1}-d-X_t+d \mid X_0, \ldots , X_t) \\
    &= \E(X_{t+1}-X_t \mid X_0, \ldots , X_t)  = \sum_{j=-3}^3 j \cdot p_j = 0.\\
    \intertext{On the other hand, if $X_t < d$,}
    & \E(|X_{t+1}-d|-|X_t-d| \mid X_0, \ldots , X_t) \\ 
    &\geq \E(|X_{t+1}-d| +X_t-d \mid X_0, \ldots , X_t)\\
    & \geq \E(-X_{t+1}+d +X_t-d \mid X_0, \ldots , X_t) \\
    &= \E(X_t-X_{t+1} \mid X_0, \ldots , X_t) = 0.
    \end{align*}
    Further, 
    \begin{align*}
    \text{Var}(Y_{t+1}-Y_t \mid Y_0, \ldots , Y_t) \geq \E((Y_{t+1}-Y_t)^2 \mid Y_0, \ldots , Y_t) \geq \delta
    \end{align*}
    for $\delta:= 1/(en)$ since flipping a specific bit (the first zero from the left if $X_t \geq d$ to increase $X_t$ and otherwise the first one from the right to decrease $X_t$) changes $Y_t$ by one. This happens with probability at least $(1-1/n)^{n-1} \cdot 1/n \geq 1/(en)$. So both conditions of Theorem~\ref{thm:unbiased-random-walk} are satisfied and we obtain $\E(T') \leq n^2/\delta = O(n^3)$ for 
    $$T':=\inf \left\{t \geq 0 \mid Y_t > \frac{b^{\text{right}}-b^{\text{left}}}{2} - 4 \right\}:$$
    We see that $b^{\text{left}}+4 \leq X_t \leq b^{\text{right}}-4$ iff $(b^{\text{left}}-b^{\text{right}})/2+4 \leq X_t - d \leq (b^{\text{right}}-b^{\text{left}})/2-4$ iff $Y_t \leq (b^{\text{right}}-b^{\text{left}})/2-4$. This also yields $T=T'$.  

    \textbf{Case 2:} Suppose that $b^{\text{right}} \neq \infty$ and $b^{\text{left}} = -\infty$. Then we have $T= \inf\{t \geq 0 \mid X_t > b^{\text{right}} - 4\}$. If $X_t \geq 3$ we obtain
    \begin{align*}
    \E(X_{t+1}-X_t \mid X_0, \ldots , X_t) = \sum_{j=-3}^3 j \cdot p_j = 0
    \end{align*}
    since $3 \leq X_t \leq b^{\text{right}}-4$. Further, if $X_t = j$ for $j \in \{0,1,2\}$, then $X_t$ cannot decrease by more than $j$ and hence, we obtain
    \begin{align*}
    \E(X_{t+1}-X_t \mid X_0, \ldots , X_t) = \sum_{i=-j}^j i \cdot p_i + \sum_{i=j+1}^3 i \cdot p_i =  \sum_{i=j+1}^3 i \cdot p_i > 0.
    \end{align*}
    As above, $\delta:=\text{Var}(X_{t+1}-X_t \mid X_0, \ldots , X_t) = \Omega(1/n)$. So with Theorem~\ref{thm:unbiased-random-walk} we obtain $\E(T) \leq n^2/\delta = O(n^3)$
    since $b^{\text{right}}-4 \in O(n)$.

    \textbf{Case 3:} Suppose that $b^{\text{right}} = \infty$ and $b^{\text{left}} \neq -\infty$. Let $Y_t:=n-X_t$. Then $T=\inf\{t \geq 0 \mid X_t < b^{\text{left}}-4\} = \inf\{t \geq 0 \mid Y_t > n-b^{\text{left}}+4\}$, and similar as in Case~2, $\E(Y_{t+1}-Y_t \mid Y_0, \ldots , Y_t) \geq 0$ and $\text{Var}(Y_{t+1}-Y_t \mid Y_0, \ldots , Y_t) = \Omega(1/n)$. Therefore, $\E(T) = O(n^3)$ with Theorem~\ref{thm:unbiased-random-walk}. All the above cases conclude the proof of Claim~\ref{cl:exit-smaller-interval}.
    \end{proofofclaim}

   We show in the next claim that $\LO(s)$ even exits $I:=[b^{\text{left}} +1,b^{\text{right}} - 1]$ in expected $O(n^3)$ iterations if no failure occurs. This is because, if $\LO(s) \notin [b^{\text{left}} +4,b^{\text{right}} - 4]$, then $\LO(s)$ only needs to be changed by three specific one-bit flips in the next three iterations when the fitness of $s$ changes. 
    
    \begin{claim}
    \label{cl:exit-larger-interval}
    After $O(n^3)$ iterations in expectation a failure occurs or $\LO(s) \notin [b^{\text{left}} +1,b^{\text{right}} - 1]$ is satisfied for every possible $b^{\text{left}}<b^{\text{right}}$ with $b^{\text{left}} \neq -\infty$ or $b^{\text{right}} \neq \infty$, and for every possible starting point $s_0$ with $\LO(s_0) \in [b^{\text{left}} +1,b^{\text{right}} - 1]$.
    \end{claim}

    \begin{proofofclaim}
    As in the proof of the claim before, we may assume that no failure occurs. Denote an iteration as \emph{profitable} if a candidate solution $c$ is created with $\LO(c) = \LO(s)+1$ if $\LO(s) > b^{\text{right}} - 4$ or with $\LO(c)=\LO(s)-1$ if $\LO(s) < b^{\text{left}} + 4$. Note that if $\LO(s) \notin [b^{\text{left}} +4,b^{\text{right}} - 4]$, then also $\LO(c) \notin [b^{\text{left}} +4,b^{\text{right}} - 4]$ if $t$ is profitable.
    With Claim~\ref{cl:exit-smaller-interval} we can assume that $\LO(s) \notin [b^{\text{left}} +4,b^{\text{right}} - 4]$. Suppose that $\LO(s) > b^{\text{right}} - 4$ (i.e. $b^{\text{right}} \neq \infty$). To exit the interval $[b^{\text{left}}+1,b^{\text{right}} - 1]$, it is sufficient that the next $j = b^{\text{right}} - \LO(s) \leq 3$ iterations where the fitness of $s$ changes (i.e. the candidate solution $c$ with $f(c) \neq f(s)$ becomes the current solution for the next iteration) are profitable. In the following we estimate the probability that an iteration where the fitness of $s$ changes is profitable. The probability that the fitness of $s$ changes is at most $2/n$  since it is required to either flip the first zero from the left (to increase $\LO(s)$ but decrease $\TZ(s)$) or the last one from the left (to decrease $\LO(s)$ but increase $\TZ(s)$). The probability to perform a specific one-bit flip is $1/n \cdot (1-1/n)^{n-1} \geq 1/(en)$. Hence, as long as $s \in [b^{\text{left}} +1,b^{\text{right}} - 1]$, an iteration where the fitness changes is profitable with probability at least $1/(en) \cdot 1/(2/n) = 1/(2e) = \Omega(1)$ and hence, the next $j \leq 3$ iterations where the fitness changes are profitable with probability $\Omega(1)$. If this happens, $s$ exits $I$. By symmetry, the same also holds if $\LO(s) < b^{\text{left}} + 4$.
    
    Denote the next $j$ iterations during which the fitness of $s$ changes, as a \emph{trial}, and call the trial \emph{successful} if the interval $[b^{\text{left}} +1,b^{\text{right}} - 1]$ is exited by $s$. Note that the expected number of iterations until $s$ exits $I$ is $O(n)$ if the first trial is successful since the expected waiting time for a fitness change of $s$ is $O(n)$. If the first trial is not successful, we repeat the arguments above (possibly those included in Claim~\ref{cl:exit-smaller-interval}). Since the probability that the first trial is successful is $\Omega(1)$, we see that the expected number of repetitions until a successful trial occurs is $O(1)$. This proves Claim~\ref{cl:exit-larger-interval}. 
    \end{proofofclaim}

    Claim~\ref{cl:exit-larger-interval} also shows that in expected $O(n^3)$ iterations a Pareto-optimal search point $x$ with $x \notin A_t$ is added to $A_t$ or a failure occurs. Now we estimate the probability that within $O(n^3)$ iterations a failure occurs. The probability that $\LO(s)$ changes by at least four if $s$ is Pareto-optimal is $O(1/n^4)$ since it is required to flip a specific block of four ones or four zeros. By a union bound, the probability is $o(1)$ that such a change occurs in $O(n^3)$ iterations. Further, we argue that a new Pareto-optimal solution is created with probability $\Omega(1)$ when $\LO(s)$ exits the interval $[b^{\text{left}}+1,b^{\text{right}}-1]$. Let $j$ be the smallest possible Hamming distance between $s$ and a Pareto-optimal solution $x$ not already in $A_t$. Then the probability that $c=x$ is at least $1/(en^j)$ while $s$ exits the interval with probability at most $2/n^j$ since it is required to either increase $\LO(s)$ by at least $j$ or decrease $\LO(s)$ by at least $j$. Otherwise, $s$ remains Pareto-optimal. Hence, with probability $\Omega(1)$ the candidate solution $c$ is a new Pareto-optimal solution if $s$ exits $I$. This yields that a failure does not occur with probability $\Omega(1)$ by a union bound on both possible failures within the $O(n^3)$ iterations from Claim~\ref{cl:exit-larger-interval} above. If a failure occurs within these $O(n^3)$ iterations we repeat the above arguments, including those that $s$ reaches the Pareto front and then exiting the interval $[b^{\text{left}}-1,b^{\text{right}}+1]$. The expected number of repetitions is $O(1)$ concluding the proof. 
\end{proof}


\section{\PAES Fails on \OMM and \COCZ}

In this section, we demonstrate that \PAES using either one-bit mutation or standard bit mutation performs poorly on other pseudo-Boolean benchmarks such as \OMM or \COCZ. This inefficiency arises from the update mechanism. The current solution is replaced by a candidate solution $s$ whenever $s$ is added to the archive $A_t$ and $A_t$ is not full. Our theoretical results and proofs naturally extend to the many-objective setting if the number of objectives is constant. Recall that \OMM~\cite{Giel2010} is defined as $\OMM(x)=(\sum_{i=1}^n x_i,\sum_{i=1}^n (1-x_i))$, where the first component counts the number of ones in $x$, and the second counts the number of zeros. The function \COCZ~\cite{Laumanns2004} is defined as $\COCZ(x)=(\sum_{i=1}^n x_i,\sum_{i=1}^{n/2} x_i + \sum_{i=n/2+1}^n (1-x_i))$. Here, the first component again counts the number of ones, while the second component counts the number of ones in the first half and the number of zeros in the second half of the bit string.

\begin{theorem}
\label{thm:spreading-Pareto-front-omm}
    Consider \PAES with standard bit mutation or one-bit mutation and archive size at least $n+1$ optimizing \OMM. Let $0 < \alpha \leq 1/2$ be a constant. Then there is a constant $\beta>0$ such that with probability $1-e^{-\Omega(n)}$ the value $\ones{s}$ where $s$ is the current solution does not exit the interval $[n/2-\alpha n,n/2+\alpha n]$ within $2^{\beta n}$ iterations. Hence, with probability $1-e^{-\Omega(n)}$ the fraction of the Pareto front contained in $A_t$ is $o(n)$ after $2^{\Omega(n)}$ iterations. 
\end{theorem}

\begin{proof}
    By a classical Chernoff bound, we have that $n/2-\alpha n/2 \leq s \leq n/2+\alpha n/2$ with probability $1-e^{-\Omega(n)}$ in the first iteration. Suppose that this happens. Note that a new created candidate solution is either non-dominated or has the same fitness value as a solution $x \in A_t$. In both cases it is added to $A_t$. Hence, the current solution is updated in each iteration no matter how many bits in $s$ are flipped. Denote by $X_t$ the number of ones in $s$ in iteration $t$. Then we obtain for $a:=n/2-\alpha n$ and $b:=n/2-\alpha n/2$
    $$\E(X_t-X_{t+1} \mid X_0, \ldots , X_t; a < X_t < b) \leq \frac{b}{n} - \frac{n-b}{n}   = -\frac{n-2b}{n} = -\alpha.$$  
    Hence by the negative drift theorem (see for example Theorem 2.4.9 in~\cite{Lengler2020}) there is a constant $\beta>0$ such that $X_t<a$ within the next $2^{\beta n}$ iterations with probability $e^{-\Omega(n)}$. By changing the roles of ones and zeros we also see that $X_t > n/2+\alpha n$ within $2^{\beta n}$ iterations with probability $e^{-\Omega(n)}$. By a union bound we obtain that $n/2 - \alpha/2 \leq X_t \leq n/2+\alpha n/2$ within the next $2^{\beta n}$ iterations with probability $1-e^{-\Omega(n)}$ which concludes the proof.
\end{proof}

This negative result directly transfers to \ojzjk \, for $2 \leq k \leq n/2-cn$ for a constant $0<c<1/2$ since it coincides with \OMM if $n/2-cn \leq \ones{s} \leq n/2+c n$ up to the additive vector $(k,k)$. We formulate a similar result for \COCZ. For a bit string $x \in \{0,1\}^n$ denote by $x^1$ its first half and by $x^2$ its second if $n$ is divisible by two. 

\begin{theorem}
\label{thm:spreading-Pareto-front-cocz}
    Suppose that $n$ is divisible by two and consider \PAES with standard bit mutation or one-bit mutation and archive size at least $n/2+1$ optimizing \COCZ. Let $0 < \alpha \leq 1/4$ be a constant. Then there is a constant $\beta>0$ such that with probability $1-e^{-\Omega(n)}$ the value $\ones{s^2}$ where $s$ denotes the current solution does not exit the interval $[n/4-\alpha n,n/4+\alpha n]$ within $2^{\beta n}$ iterations. Hence, with probability $1-e^{-\Omega(n)}$ the fraction of the Pareto front contained in $A_t$ is $o(n)$ after $2^{\Omega(n)}$ iterations. 
\end{theorem}

\begin{proof}
We can assume that $\alpha<1/4$ since $\ones{s^2}$ never exits the interval $[0,n/2]$. By a classical Chernoff bound, we have that $n/4-\alpha n/2 \leq \ones{s^2} \leq n/4+\alpha n/2$ with probability $1-e^{-\Omega(n)}$. Suppose that this happens. Let $B$ be the event that no bit is flipped in the first half of the bit string during mutation. If $B$ occurs, a new created candidate solution is either non-dominated or has the same fitness value as a solution $x \in A_t$ and we are in the same situation as in the proof of Theorem~\ref{thm:spreading-Pareto-front-omm}. Denote by $X_t$ the number of ones in the second half of the current solution in iteration $t$. Then we obtain for $a:=n/4-\alpha n$ and $b:=n/4-\alpha n/2$ also 
    $$\E(X_t-X_{t+1} \mid X_0, \ldots , X_t; a < X_t < b;B) \leq \frac{b}{n}-\frac{n/2-b}{n} = -\alpha.$$
    Further we obtain 
    $\E(X_t-X_{t+1} \mid X_0, \ldots , X_t; a < X_t < b;\bar{B}) \leq n^{-\Omega(n)}$: If the offspring was not accepted then the latter becomes zero. Suppose that $s$ is accepted by either flipping $k$ ones and $\ell$ zeros or flipping $\ell$ ones and $k$ zeros for $k,\ell \leq n/4-\alpha n$ with $k<\ell$ in the second half. The probability for the first event is larger than the second and hence, conditioned on the union of both events yields 
    $\E(X_t-X_{t+1} \mid X_0, \ldots , X_t; a < X_t < b;\bar{B}) \leq 0$.
    Hence, by the law of total probability, $\E(X_t-X_{t+1} \mid X_0, \ldots , X_t; a < X_t < b;\bar{B}) \leq 0$ if at most $n/4-\alpha n$ many ones and zeros are flipped. Note that the probability to flip more than $n/4-\alpha n$ ones or zeros is $n^{-\Omega(n)}$ (since $n/4-\alpha n = \Omega(n)$) which can increase $X_t$ by at most $n/2-a=n/4+\alpha n/2 = O(n)$. Hence, in total again by the law of total probability    
    $$\E(X_t-X_{t+1} \mid X_0, \ldots , X_t; a < X_t < b;\bar{B}) \leq n^{-\Omega(n)}.$$
    Again, applying the law of total probability yields due to $\Pr(B) = (1-1/n)^{n/2} \geq 1/e$
    \begin{align*}
    & \E(X_t-X_{t+1} \mid X_0, \ldots , X_t; a < X_t < b) \\
    &\leq \E(X_t-X_{t+1} \mid X_0, \ldots , X_t; a < X_t < b;B) \Pr(B) + n^{-\Omega(n)} \leq -\frac{\alpha}{2e}
    \end{align*}
    for $n$ sufficiently large. Hence, again by the negative drift theorem, there is a constant $\beta>0$ such that $X_t<a$ within the next $2^{\beta n}$ iterations with probability $e^{-\Omega(n)}$. By changing the roles of ones and zeros we also see that $X_t > n/4+\alpha n/2$ within $2^{\beta n}$ iterations with probability $e^{-\Omega(n)}$. By a union bound we obtain that $n/4 - \alpha n/2 \leq X_t \leq n/4+\alpha n/2$ within the next $2^{\beta n}$ iterations with probability $1-e^{-\Omega(n)}$ which concludes the proof.
\end{proof}

\section{Conclusions}
In this paper, we rigorously analyzed the performance of \PAES on the classical benchmark functions \mLOTZ, \COCZ, and \OMM. Notably, \PAES with one-bit mutation performs surprisingly well on \mLOTZ, and we were able to derive tight runtime bounds in the many-objective setting. To the best of our knowledge, these are the first tight bounds established for an MOEA on a classical benchmark problem. Further, our analysis reveals new insights into the behavior of local search-based algorithms on classical pseudo-Boolean benchmarks. In particular, we showed that the time of \PAES with one-bit mutation finding the Pareto front of \mLOTZ is, up to a factor of $n/m$, asymptotically the same as the cover time of a simple random walk on an $m/2$-dimensional grid graph with $(2n/m)^{m/2}$ nodes, where the latter is well known. Further, \PAES with one-bit mutation achieves a good distribution of the archive across the Pareto front when the Adaptive Grid Archiver (AGA) is used, even when the archive size is smaller than the size of the Pareto front. In the bi-objective setting, good distributions for small archive sizes are also achieved when \PAES uses the Hypervolume based Archiver (HVA) or the Multi-level Grid Archiver (MGA). A further contribution is that we derived upper runtime bounds for \PAES on \LOTZ for standard bit mutation. A major challenge in this analysis is that the current solution does not necessarily remain on the Pareto front. We also demonstrated that \PAES has still some limitations. It does not perform well on \COCZ and \OMM, since it covers only a sublinear fraction of the Pareto front in $e^{\Omega(n)}$ iterations with probaility $1-e^{\Omega(n)}$.

We hope that the techniques developed mark a significant step forward in understanding the behavior of \PAES, while also shedding light on its strengths and limitations. We are confident that our analysis carries over to other MOEAs with similar working principles as the PAES-25 like only mutating a current solution instead of choosing a parent uniformly at random. To deepen this understanding, future work could focus on more variants of \PAES which follow a $(\mu+\lambda)$ or $(\mu,\lambda)$ scheme to possibly overcome the difficulties with \OMM or \COCZ. Another research direction is to extend our analysis to more complex benchmark problems, such as FRITZ, LITZ or sLITZ proposed in~\cite{Knowles2025}, where \PAES has also been shown to perform empirically well. We also hope that the gained deeper theoretical insights into the dynamic of \PAES will provide valuable guidance for practitioners, helping to develop refined versions of the algorithm, especially for problems with non-flat fitness landscapes or even those featuring many local optima.

\balance
\bibliographystyle{ACM-Reference-Format}
\bibliography{references,alles_ea_master,ich_master,aaai25}


\begin{thebibliography}{53}


\ifx \showCODEN    \undefined \def \showCODEN     #1{\unskip}     \fi
\ifx \showDOI      \undefined \def \showDOI       #1{#1}\fi
\ifx \showISBNx    \undefined \def \showISBNx     #1{\unskip}     \fi
\ifx \showISBNxiii \undefined \def \showISBNxiii  #1{\unskip}     \fi
\ifx \showISSN     \undefined \def \showISSN      #1{\unskip}     \fi
\ifx \showLCCN     \undefined \def \showLCCN      #1{\unskip}     \fi
\ifx \shownote     \undefined \def \shownote      #1{#1}          \fi
\ifx \showarticletitle \undefined \def \showarticletitle #1{#1}   \fi
\ifx \showURL      \undefined \def \showURL       {\relax}        \fi
\providecommand\bibfield[2]{#2}
\providecommand\bibinfo[2]{#2}
\providecommand\natexlab[1]{#1}
\providecommand\showeprint[2][]{arXiv:#2}

\bibitem[Auger et~al\mbox{.}(2012)]%
        {AUGER201275}
\bibfield{author}{\bibinfo{person}{Anne Auger}, \bibinfo{person}{Johannes
  Bader}, \bibinfo{person}{Dimo Brockhoff}, {and} \bibinfo{person}{Eckart
  Zitzler}.} \bibinfo{year}{2012}\natexlab{}.
\newblock \showarticletitle{Hypervolume-based multiobjective optimization:
  Theoretical foundations and practical implications}.
\newblock \bibinfo{journal}{\emph{Theoretical Computer Science}}
  \bibinfo{volume}{425} (\bibinfo{year}{2012}), \bibinfo{pages}{75--103}.
\newblock


\bibitem[Bian et~al\mbox{.}(2018)]%
        {BianQT18ijcaigeneral}
\bibfield{author}{\bibinfo{person}{Chao Bian}, \bibinfo{person}{Chao Qian},
  {and} \bibinfo{person}{Ke Tang}.} \bibinfo{year}{2018}\natexlab{}.
\newblock \showarticletitle{A general approach to running time analysis of
  multi-objective evolutionary algorithms}. In
  \bibinfo{booktitle}{\emph{International Joint Conference on Artificial
  Intelligence}} \emph{(\bibinfo{series}{IJCAI 2018})}.
  \bibinfo{publisher}{{ijcai.org}}, \bibinfo{pages}{1405--1411}.
\newblock


\bibitem[Bian et~al\mbox{.}(2024)]%
        {ArchiveMOEAs}
\bibfield{author}{\bibinfo{person}{Chao Bian}, \bibinfo{person}{Shengjie Ren},
  \bibinfo{person}{Miqing Li}, {and} \bibinfo{person}{Chao Qian}.}
  \bibinfo{year}{2024}\natexlab{}.
\newblock \showarticletitle{An archive can bring provable speed-ups in
  multi-objective evolutionary algorithms}. In
  \bibinfo{booktitle}{\emph{Proceedings of the Thirty-Third International Joint
  Conference on Artificial Intelligence}} \emph{(\bibinfo{series}{IJCAI
  2024})}. \bibinfo{publisher}{ijcai.org}, \bibinfo{pages}{6905 -- 6913}.
\newblock


\bibitem[Bian et~al\mbox{.}(2023)]%
        {UpBian}
\bibfield{author}{\bibinfo{person}{Chao Bian}, \bibinfo{person}{Yawen Zhou},
  \bibinfo{person}{Miqing Li}, {and} \bibinfo{person}{Chao Qian}.}
  \bibinfo{year}{2023}\natexlab{}.
\newblock \showarticletitle{Stochastic population update can provably be
  helpful in multi-objective evolutionary algorithms}. In
  \bibinfo{booktitle}{\emph{Proceedings of the Thirty-Second International
  Joint Conference on Artificial Intelligence}} \emph{(\bibinfo{series}{IJCAI
  2023})}. \bibinfo{pages}{5513--5521}.
\newblock


\bibitem[Bringmann and Friedrich(2010)]%
        {HYPERALGO}
\bibfield{author}{\bibinfo{person}{Karl Bringmann} {and}
  \bibinfo{person}{Tobias Friedrich}.} \bibinfo{year}{2010}\natexlab{}.
\newblock \showarticletitle{An Efficient Algorithm for Computing Hypervolume
  Contributions}.
\newblock \bibinfo{journal}{\emph{Evolutionary Computation}}
  \bibinfo{volume}{18}, \bibinfo{number}{3} (\bibinfo{year}{2010}),
  \bibinfo{pages}{383--402}.
\newblock


\bibitem[Cerf et~al\mbox{.}(2023)]%
        {DBLP:conf/ijcai/CerfDHKW23}
\bibfield{author}{\bibinfo{person}{Sacha Cerf}, \bibinfo{person}{Benjamin
  Doerr}, \bibinfo{person}{Benjamin Hebras}, \bibinfo{person}{Yakob Kahane},
  {and} \bibinfo{person}{Simon Wietheger}.} \bibinfo{year}{2023}\natexlab{}.
\newblock \showarticletitle{The First Proven Performance Guarantees for the
  Non-Dominated Sorting Genetic Algorithm {II} {(NSGA-II)} on a Combinatorial
  Optimization Problem}. In \bibinfo{booktitle}{\emph{Proceedings of the
  International Joint Conference on Artificial Intelligence}}
  \emph{(\bibinfo{series}{IJCAI~2023})}. \bibinfo{publisher}{ijcai.org},
  \bibinfo{pages}{5522--5530}.
\newblock


\bibitem[Coello et~al\mbox{.}(2013)]%
        {coello2013evolutionary}
\bibfield{author}{\bibinfo{person}{C Coello}, \bibinfo{person}{D~Van
  Veldhuizen}, {and} \bibinfo{person}{G Lamont}.}
  \bibinfo{year}{2013}\natexlab{}.
\newblock \bibinfo{booktitle}{\emph{Evolutionary Algorithms for Solving
  Multi-Objective Problems}}.
\newblock \bibinfo{publisher}{Springer US}.
\newblock


\bibitem[Dang et~al\mbox{.}(2023)]%
        {DaOp2023}
\bibfield{author}{\bibinfo{person}{Duc-Cuong Dang}, \bibinfo{person}{Andre
  Opris}, \bibinfo{person}{Bahare Salehi}, {and} \bibinfo{person}{Dirk
  Sudholt}.} \bibinfo{year}{2023}\natexlab{}.
\newblock \showarticletitle{Analysing the Robustness of {NSGA-II} under Noise}.
  In \bibinfo{booktitle}{\emph{Proceedings of the Genetic and Evolutionary
  Computation Conference ({GECCO}'23)}}. \bibinfo{publisher}{{ACM} Press},
  \bibinfo{pages}{642--651}.
\newblock


\bibitem[Dang et~al\mbox{.}(2024a)]%
        {Dang2024}
\bibfield{author}{\bibinfo{person}{Duc-Cuong Dang}, \bibinfo{person}{Andre
  Opris}, {and} \bibinfo{person}{Dirk Sudholt}.}
  \bibinfo{year}{2024}\natexlab{a}.
\newblock \showarticletitle{Crossover can guarantee exponential speed-ups in
  evolutionary multi-objective optimisation}.
\newblock \bibinfo{journal}{\emph{Artificial Intelligence}}
  \bibinfo{volume}{330} (\bibinfo{year}{2024}), \bibinfo{pages}{104098}.
\newblock
\showISSN{0004-3702}


\bibitem[Dang et~al\mbox{.}(2024b)]%
        {DangEfficient2024}
\bibfield{author}{\bibinfo{person}{Duc-Cuong Dang}, \bibinfo{person}{Andre
  Opris}, {and} \bibinfo{person}{Dirk Sudholt}.}
  \bibinfo{year}{2024}\natexlab{b}.
\newblock \showarticletitle{Illustrating the Efficiency of Popular Evolutionary
  Multi-Objective Algorithms Using Runtime Analysis}. In
  \bibinfo{booktitle}{\emph{Proceedings of the Genetic and Evolutionary
  Computation Conference}} \emph{(\bibinfo{series}{GECCO 2024})}.
  \bibinfo{publisher}{ACM Press}, \bibinfo{pages}{484–492}.
\newblock


\bibitem[Dang et~al\mbox{.}(2024c)]%
        {Dang-Levels}
\bibfield{author}{\bibinfo{person}{Duc-Cuong Dang}, \bibinfo{person}{Andre
  Opris}, {and} \bibinfo{person}{Dirk Sudholt}.}
  \bibinfo{year}{2024}\natexlab{c}.
\newblock \showarticletitle{Level-Based Theorems for Runtime Analysis of
  Multi-objective Evolutionary Algorithms}. In
  \bibinfo{booktitle}{\emph{Proceedings of the International Conference on
  Parallel Problem Solving from Nature}} \emph{(\bibinfo{series}{PPSN XVIII})}.
  \bibinfo{publisher}{Springer}, \bibinfo{address}{Berlin, Heidelberg},
  \bibinfo{pages}{246–263}.
\newblock


\bibitem[Deb(2001)]%
        {kdeb01}
\bibfield{author}{\bibinfo{person}{Kalyanmoy Deb}.}
  \bibinfo{year}{2001}\natexlab{}.
\newblock \bibinfo{booktitle}{\emph{Multi-Objective Optimization using
  Evolutionary Algorithms}}.
\newblock \bibinfo{publisher}{John Wiley \& Sons}.
\newblock


\bibitem[Deb et~al\mbox{.}(2002)]%
        {DebPAM02}
\bibfield{author}{\bibinfo{person}{Kalyanmoy Deb}, \bibinfo{person}{Amrit
  Pratap}, \bibinfo{person}{Sameer Agarwal}, {and} \bibinfo{person}{T.
  Meyarivan}.} \bibinfo{year}{2002}\natexlab{}.
\newblock \showarticletitle{A fast and elitist multiobjective genetic
  algorithm: {NSGA-II}}.
\newblock \bibinfo{journal}{\emph{IEEE Transactions on Evolutionary
  Computation}}  \bibinfo{volume}{6} (\bibinfo{year}{2002}),
  \bibinfo{pages}{182--197}.
\newblock


\bibitem[Deng et~al\mbox{.}(2024)]%
        {MOEASubset}
\bibfield{author}{\bibinfo{person}{Renzhong Deng}, \bibinfo{person}{Weijie
  Zheng}, \bibinfo{person}{Mingfeng Li}, \bibinfo{person}{Jie Liu}, {and}
  \bibinfo{person}{Benjamin Doerr}.} \bibinfo{year}{2024}\natexlab{}.
\newblock \showarticletitle{Runtime Analysis for State-of-the-Art
  Multi-objective Evolutionary Algorithms on the Subset Selection Problem}. In
  \bibinfo{booktitle}{\emph{Proceedings of the International Conference on
  Parallel Problem Solving from Nature}} \emph{(\bibinfo{series}{PPSN XVIII})}.
  \bibinfo{publisher}{Springer}, \bibinfo{pages}{264–279}.
\newblock


\bibitem[Doerr et~al\mbox{.}(2025a)]%
        {Krejca2025b}
\bibfield{author}{\bibinfo{person}{Benjamin Doerr}, \bibinfo{person}{Tudor
  Ivan}, {and} \bibinfo{person}{Martin~S. Krejca}.}
  \bibinfo{year}{2025}\natexlab{a}.
\newblock \showarticletitle{Speeding Up the {NSGA-II} With a Simple
  Tie-Breaking Rule}. In \bibinfo{booktitle}{\emph{Proceedings of the {AAAI}
  Conference on Artificial Intelligence}} \emph{(\bibinfo{series}{AAAI~2025})}.
  \bibinfo{publisher}{{AAAI} Press}, \bibinfo{pages}{26964--26972}.
\newblock


\bibitem[Doerr et~al\mbox{.}(2025b)]%
        {doerr2025tightruntimeguaranteesunderstanding}
\bibfield{author}{\bibinfo{person}{Benjamin Doerr}, \bibinfo{person}{Martin
  Krejca}, {and} \bibinfo{person}{Andre Opris}.}
  \bibinfo{year}{2025}\natexlab{b}.
\newblock \bibinfo{title}{Tight Runtime Guarantees From Understanding the
  Population Dynamics of the GSEMO Multi-Objective Evolutionary Algorithm}.
\newblock
\newblock
\showeprint[arxiv]{2505.01266}~[cs.NE]
\newblock
\shownote{to appear at IJCAI 2025}.


\bibitem[Doerr et~al\mbox{.}(2025c)]%
        {Doerr_Krejca_Rudolph_2025}
\bibfield{author}{\bibinfo{person}{Benjamin Doerr}, \bibinfo{person}{Martin~S.
  Krejca}, {and} \bibinfo{person}{Günter Rudolph}.}
  \bibinfo{year}{2025}\natexlab{c}.
\newblock \showarticletitle{Runtime Analysis for Multi-Objective Evolutionary
  Algorithms in Unbounded Integer Spaces}.
\newblock \bibinfo{journal}{\emph{Proceedings of the AAAI Conference on
  Artificial Intelligence (AAAI 2025)}}, \bibinfo{pages}{26955--26963}.
\newblock


\bibitem[Doerr and Qu(2022)]%
        {Qu2022PPSN}
\bibfield{author}{\bibinfo{person}{Benjamin Doerr} {and}
  \bibinfo{person}{Zhongdi Qu}.} \bibinfo{year}{2022}\natexlab{}.
\newblock \showarticletitle{A First Runtime Analysis of the {NSGA-II} on a
  Multimodal Problem}. In \bibinfo{booktitle}{\emph{Proceedings of the
  International Conference on Parallel Problem Solving from Nature}}
  \emph{(\bibinfo{series}{PPSN XVII})}. \bibinfo{publisher}{Springer},
  \bibinfo{pages}{399--412}.
\newblock


\bibitem[Doerr and Qu(2023a)]%
        {DoerrQu2023a}
\bibfield{author}{\bibinfo{person}{Benjamin Doerr} {and}
  \bibinfo{person}{Zhongdi Qu}.} \bibinfo{year}{2023}\natexlab{a}.
\newblock \showarticletitle{From Understanding the Population Dynamics of the
  {NSGA-II} to the First Proven Lower Bounds}. In
  \bibinfo{booktitle}{\emph{Proceedings of the AAAI Conference on Artificial
  Intelligence}} \emph{(\bibinfo{series}{AAAI 2023})}.
  \bibinfo{publisher}{{AAAI} Press}, \bibinfo{pages}{12408--12416}.
\newblock


\bibitem[Doerr and Qu(2023b)]%
        {DoerrQu2022}
\bibfield{author}{\bibinfo{person}{Benjamin Doerr} {and}
  \bibinfo{person}{Zhongdi Qu}.} \bibinfo{year}{2023}\natexlab{b}.
\newblock \showarticletitle{Runtime Analysis for the {NSGA-II}: Provable
  Speed-Ups From Crossover}. In \bibinfo{booktitle}{\emph{Proceedings of the
  AAAI Conference on Artificial Intelligence ({AAAI}~2023)}}.
  \bibinfo{publisher}{{AAAI} Press}, \bibinfo{pages}{12399--12407}.
\newblock


\bibitem[Giel(2003)]%
        {Giel03}
\bibfield{author}{\bibinfo{person}{Oliver Giel}.}
  \bibinfo{year}{2003}\natexlab{}.
\newblock \showarticletitle{Expected runtimes of a simple multi-objective
  evolutionary algorithm}. In \bibinfo{booktitle}{\emph{Congress on
  Evolutionary Computation, {CEC} 2003}}. \bibinfo{publisher}{{IEEE}},
  \bibinfo{pages}{1918--1925}.
\newblock


\bibitem[Giel and Lehre(2010)]%
        {Giel2010}
\bibfield{author}{\bibinfo{person}{Oliver Giel} {and}
  \bibinfo{person}{Per~Kristian Lehre}.} \bibinfo{year}{2010}\natexlab{}.
\newblock \showarticletitle{On the Effect of Populations in Evolutionary
  Multi-Objective Optimisation}.
\newblock \bibinfo{journal}{\emph{Evolutionary Computation}}
  \bibinfo{volume}{18}, \bibinfo{number}{3} (\bibinfo{year}{2010}),
  \bibinfo{pages}{335--356}.
\newblock


\bibitem[Ikeda et~al\mbox{.}(2009)]%
        {IKEDA200994}
\bibfield{author}{\bibinfo{person}{Satoshi Ikeda}, \bibinfo{person}{Izumi
  Kubo}, {and} \bibinfo{person}{Masafumi Yamashita}.}
  \bibinfo{year}{2009}\natexlab{}.
\newblock \showarticletitle{The hitting and cover times of random walks on
  finite graphs using local degree information}.
\newblock \bibinfo{journal}{\emph{Theoretical Computer Science}}
  (\bibinfo{year}{2009}), \bibinfo{pages}{94--100}.
\newblock


\bibitem[Jonasson(2000)]%
        {JONASSON2000181}
\bibfield{author}{\bibinfo{person}{Johan Jonasson}.}
  \bibinfo{year}{2000}\natexlab{}.
\newblock \showarticletitle{An upper bound on the cover time for powers of
  graphs}.
\newblock \bibinfo{journal}{\emph{Discrete Mathematics}} \bibinfo{volume}{222},
  \bibinfo{number}{1} (\bibinfo{year}{2000}), \bibinfo{pages}{181--190}.
\newblock


\bibitem[Knowles(2002)]%
        {KnowlesPhD}
\bibfield{author}{\bibinfo{person}{Joshua Knowles}.}
  \bibinfo{year}{2002}\natexlab{}.
\newblock \emph{\bibinfo{title}{Local-Search and Hybrid Evolutionary Algorithms
  for Pareto Optimization}}.
\newblock \bibinfo{thesistype}{Ph.\,D. Dissertation}. \bibinfo{school}{The
  University of Reading}, \bibinfo{address}{United Kingdom}.
\newblock


\bibitem[Knowles and Corne(1999)]%
        {781913}
\bibfield{author}{\bibinfo{person}{Joshua Knowles} {and} \bibinfo{person}{David
  Corne}.} \bibinfo{year}{1999}\natexlab{}.
\newblock \showarticletitle{The Pareto archived evolution strategy: a new
  baseline algorithm for Pareto multiobjective optimisation}. In
  \bibinfo{booktitle}{\emph{Proceedings of the 1999 Congress on Evolutionary
  Computation-CEC99}}. \bibinfo{pages}{98--105}.
\newblock


\bibitem[Knowles and Corne(2000)]%
        {Knowles2000}
\bibfield{author}{\bibinfo{person}{Joshua Knowles} {and} \bibinfo{person}{David
  Corne}.} \bibinfo{year}{2000}\natexlab{}.
\newblock \showarticletitle{Approximating the Nondominated Front Using the
  Pareto Archived Evolution Strategy}.
\newblock \bibinfo{journal}{\emph{Evolutionary Computation}}
  \bibinfo{volume}{8}, \bibinfo{number}{2} (\bibinfo{year}{2000}),
  \bibinfo{pages}{149--172}.
\newblock


\bibitem[Knowles and Corne(2003)]%
        {10.1109/TEVC.2003.810755}
\bibfield{author}{\bibinfo{person}{Joshua Knowles} {and} \bibinfo{person}{David
  Corne}.} \bibinfo{year}{2003}\natexlab{}.
\newblock \showarticletitle{Properties of an adaptive archiving algorithm for
  storing nondominated vectors}.
\newblock \bibinfo{journal}{\emph{Transactions of Evolutionary Computation}}
  (\bibinfo{year}{2003}), \bibinfo{pages}{100–116}.
\newblock


\bibitem[Knowles and Liefooghe(2025)]%
        {Knowles2025}
\bibfield{author}{\bibinfo{person}{Joshua Knowles} {and}
  \bibinfo{person}{Arnaud Liefooghe}.} \bibinfo{year}{2025}\natexlab{}.
\newblock \showarticletitle{PAES-25: Local Search, Archiving, and
  Multi/Many-Objective Pseudo-Boolean Functions}. In
  \bibinfo{booktitle}{\emph{Evolutionary Multi-Criterion Optimization}}.
  \bibinfo{publisher}{Springer}, \bibinfo{pages}{60--73}.
\newblock


\bibitem[Laumanns et~al\mbox{.}(2004a)]%
        {Laumanns2004}
\bibfield{author}{\bibinfo{person}{Marco Laumanns}, \bibinfo{person}{Lothar
  Thiele}, {and} \bibinfo{person}{Eckart Zitzler}.}
  \bibinfo{year}{2004}\natexlab{a}.
\newblock \showarticletitle{Running Time Analysis of Multiobjective
  Evolutionary Algorithms on Pseudo-Boolean Functions}.
\newblock \bibinfo{journal}{\emph{{IEEE} Transactions on Evolutionary
  Computation}} \bibinfo{volume}{8}, \bibinfo{number}{2}
  (\bibinfo{year}{2004}), \bibinfo{pages}{170--182}.
\newblock


\bibitem[Laumanns et~al\mbox{.}(2004b)]%
        {LaumannsTZ04}
\bibfield{author}{\bibinfo{person}{Marco Laumanns}, \bibinfo{person}{Lothar
  Thiele}, {and} \bibinfo{person}{Eckart Zitzler}.}
  \bibinfo{year}{2004}\natexlab{b}.
\newblock \showarticletitle{Running time analysis of multiobjective
  evolutionary algorithms on pseudo-{B}oolean functions}.
\newblock \bibinfo{journal}{\emph{IEEE Transactions on Evolutionary
  Computation}}  \bibinfo{volume}{8} (\bibinfo{year}{2004}),
  \bibinfo{pages}{170--182}.
\newblock


\bibitem[Laumanns et~al\mbox{.}(2002)]%
        {LaumannsTZWD02}
\bibfield{author}{\bibinfo{person}{Marco Laumanns}, \bibinfo{person}{Lothar
  Thiele}, \bibinfo{person}{Eckart Zitzler}, \bibinfo{person}{Emo Welzl}, {and}
  \bibinfo{person}{Kalyanmoy Deb}.} \bibinfo{year}{2002}\natexlab{}.
\newblock \showarticletitle{Running time analysis of multi-objective
  evolutionary algorithms on a simple discrete optimization problem}. In
  \bibinfo{booktitle}{\emph{Parallel Problem Solving from Nature, {PPSN}
  2002}}. \bibinfo{publisher}{Springer}, \bibinfo{pages}{44--53}.
\newblock


\bibitem[Laumanns and Zenklusen(2011)]%
        {LAUMANNS2011414}
\bibfield{author}{\bibinfo{person}{Marco Laumanns} {and} \bibinfo{person}{Rico
  Zenklusen}.} \bibinfo{year}{2011}\natexlab{}.
\newblock \showarticletitle{Stochastic convergence of random search methods to
  fixed size Pareto front approximations}.
\newblock \bibinfo{journal}{\emph{European Journal of Operational Research}}
  \bibinfo{volume}{213}, \bibinfo{number}{2} (\bibinfo{year}{2011}),
  \bibinfo{pages}{414--421}.
\newblock


\bibitem[Lengler(2020a)]%
        {Lengler2020Drift}
\bibfield{author}{\bibinfo{person}{Johannes Lengler}.}
  \bibinfo{year}{2020}\natexlab{a}.
\newblock \showarticletitle{Drift Analysis}.
\newblock In \bibinfo{booktitle}{\emph{Theory of Evolutionary Computation:
  Recent Developments in Discrete Optimization}}.
  \bibinfo{publisher}{Springer}, \bibinfo{pages}{89--131}.
\newblock


\bibitem[Lengler(2020b)]%
        {Lengler2020}
\bibfield{author}{\bibinfo{person}{Johannes Lengler}.}
  \bibinfo{year}{2020}\natexlab{b}.
\newblock \showarticletitle{A General Dichotomy of Evolutionary Algorithms on
  Monotone Functions}.
\newblock \bibinfo{journal}{\emph{IEEE Transactions on Evolutionary
  Computation}} \bibinfo{volume}{24}, \bibinfo{number}{6}
  (\bibinfo{year}{2020}), \bibinfo{pages}{995--1009}.
\newblock


\bibitem[L{\'o}pez-Ib{\'a}{\~{n}}ez et~al\mbox{.}(2011)]%
        {Archiver2011}
\bibfield{author}{\bibinfo{person}{Manuel L{\'o}pez-Ib{\'a}{\~{n}}ez},
  \bibinfo{person}{Joshua Knowles}, {and} \bibinfo{person}{Marco Laumanns}.}
  \bibinfo{year}{2011}\natexlab{}.
\newblock \showarticletitle{On Sequential Online Archiving of Objective
  Vectors}. In \bibinfo{booktitle}{\emph{Evolutionary Multi-Criterion
  Optimization}}. \bibinfo{publisher}{Springer Berlin Heidelberg},
  \bibinfo{pages}{46--60}.
\newblock


\bibitem[L{\'o}pez-Ib{\'a}{\~n}ez et~al\mbox{.}(2011)]%
        {LAUMANNS2011415}
\bibfield{author}{\bibinfo{person}{Manuel L{\'o}pez-Ib{\'a}{\~n}ez},
  \bibinfo{person}{Joshua Knowles}, {and} \bibinfo{person}{Marco Laumanns}.}
  \bibinfo{year}{2011}\natexlab{}.
\newblock \showarticletitle{On sequential online archiving of objective
  vectors}. In \bibinfo{booktitle}{\emph{Lecture Notes in Computer Science
  (including subseries Lecture Notes in Artificial Intelligence and Lecture
  Notes in Bioinformatics)}}. \bibinfo{publisher}{Springer Nature},
  \bibinfo{pages}{46--60}.
\newblock


\bibitem[Luukkonen et~al\mbox{.}(2023)]%
        {LUUKKONEN2023102537}
\bibfield{author}{\bibinfo{person}{Sohvi Luukkonen}, \bibinfo{person}{Helle~W.
  {van den Maagdenberg}}, \bibinfo{person}{Michael~T.M. Emmerich}, {and}
  \bibinfo{person}{Gerard~J.P. {van Westen}}.} \bibinfo{year}{2023}\natexlab{}.
\newblock \showarticletitle{Artificial Intelligence in Multi-Objective Drug
  Design}.
\newblock \bibinfo{journal}{\emph{Current Opinion in Structural Biology}}
  \bibinfo{volume}{79} (\bibinfo{year}{2023}), \bibinfo{pages}{102537}.
\newblock
\showISSN{0959-440X}


\bibitem[Opris(2025a)]%
        {opris2025multimodal}
\bibfield{author}{\bibinfo{person}{Andre Opris}.}
  \bibinfo{year}{2025}\natexlab{a}.
\newblock \bibinfo{title}{A First Runtime Analysis of NSGA-III on a
  Many-Objective Multimodal Problem: Provable Exponential Speedup via
  Stochastic Population Update}.
\newblock
\newblock
\showeprint[arxiv]{2505.01256}~[cs.NE]
\newblock
\shownote{to appear at IJCAI~2025}.


\bibitem[Opris(2025b)]%
        {Opris2025}
\bibfield{author}{\bibinfo{person}{Andre Opris}.}
  \bibinfo{year}{2025}\natexlab{b}.
\newblock \showarticletitle{A Many Objective Problem Where Crossover is
  Provably Indispensable}. In \bibinfo{booktitle}{\emph{Proceedings of the
  International Joint Conference on Artificial Intelligence}}
  \emph{(\bibinfo{series}{{AAAI} 2025})}. \bibinfo{publisher}{AAAI Press},
  \bibinfo{pages}{27108--27116}.
\newblock


\bibitem[Opris et~al\mbox{.}(2024)]%
        {OprisNSGAIII}
\bibfield{author}{\bibinfo{person}{Andre Opris}, \bibinfo{person}{Duc-Cuong
  Dang}, \bibinfo{person}{Frank Neumann}, {and} \bibinfo{person}{Dirk
  Sudholt}.} \bibinfo{year}{2024}\natexlab{}.
\newblock \showarticletitle{Runtime Analyses of NSGA-III on Many-Objective
  Problems}. In \bibinfo{booktitle}{\emph{Proceedings of the Genetic and
  Evolutionary Computation Conference ({GECCO}'24)}}. \bibinfo{publisher}{{ACM}
  Press}, \bibinfo{pages}{1596–1604}.
\newblock


\bibitem[Peng et~al\mbox{.}(2018)]%
        {PAESFIRST}
\bibfield{author}{\bibinfo{person}{Xue Peng}, \bibinfo{person}{Xiaoyun Xia},
  \bibinfo{person}{Weizhi Liao}, {and} \bibinfo{person}{Zhanwei Guo}.}
  \bibinfo{year}{2018}\natexlab{}.
\newblock \showarticletitle{Running time analysis of the Pareto archived
  evolution strategy on pseudo-Boolean functions}.
\newblock \bibinfo{journal}{\emph{Multimedia Tools Application}}
  \bibinfo{volume}{77}, \bibinfo{number}{9} (\bibinfo{year}{2018}),
  \bibinfo{pages}{11203–11217}.
\newblock


\bibitem[Qian et~al\mbox{.}(2019)]%
        {QianYTYZ19}
\bibfield{author}{\bibinfo{person}{Chao Qian}, \bibinfo{person}{Yang Yu},
  \bibinfo{person}{Ke Tang}, \bibinfo{person}{Xin Yao}, {and}
  \bibinfo{person}{Zhi{-}Hua Zhou}.} \bibinfo{year}{2019}\natexlab{}.
\newblock \showarticletitle{Maximizing submodular or monotone approximately
  submodular functions by multi-objective evolutionary algorithms}.
\newblock \bibinfo{journal}{\emph{Artificial Intelligence}}
  \bibinfo{volume}{275} (\bibinfo{year}{2019}), \bibinfo{pages}{279--294}.
\newblock


\bibitem[Qu et~al\mbox{.}(2021)]%
        {9515233}
\bibfield{author}{\bibinfo{person}{Qu Qu}, \bibinfo{person}{Zheng Ma},
  \bibinfo{person}{Anders Clausen}, {and} \bibinfo{person}{Bo~Nørregaard
  Jørgensen}.} \bibinfo{year}{2021}\natexlab{}.
\newblock \showarticletitle{A Comprehensive Review of Machine Learning in
  Multi-objective Optimization}. In \bibinfo{booktitle}{\emph{2021 IEEE 4th
  International Conference on Big Data and Artificial Intelligence (BDAI)}}.
  \bibinfo{pages}{7--14}.
\newblock


\bibitem[Ren et~al\mbox{.}(2024)]%
        {RenBLQ24}
\bibfield{author}{\bibinfo{person}{Shengjie Ren}, \bibinfo{person}{Chao Bian},
  \bibinfo{person}{Miqing Li}, {and} \bibinfo{person}{Chao Qian}.}
  \bibinfo{year}{2024}\natexlab{}.
\newblock \showarticletitle{A first running time analysis of the {S}trength
  {P}areto {E}volutionary {A}lgorithm~2 {(SPEA2)}}. In
  \bibinfo{booktitle}{\emph{Proceedings of the International Conference of
  Parallel Problem Solving from Nature}} \emph{(\bibinfo{series}{{PPSN}
  XVIII})}. \bibinfo{publisher}{Springer}, \bibinfo{pages}{295--312}.
\newblock


\bibitem[Thierens(2003)]%
        {Thierens03}
\bibfield{author}{\bibinfo{person}{Dirk Thierens}.}
  \bibinfo{year}{2003}\natexlab{}.
\newblock \showarticletitle{Convergence time analysis for the multi-objective
  counting ones problem}. In \bibinfo{booktitle}{\emph{Evolutionary
  Multi-Criterion Optimization, {EMO} 2003}}. \bibinfo{publisher}{Springer},
  \bibinfo{pages}{355--364}.
\newblock


\bibitem[Wietheger and Doerr(2023)]%
        {WiethegerD23}
\bibfield{author}{\bibinfo{person}{Simon Wietheger} {and}
  \bibinfo{person}{Benjamin Doerr}.} \bibinfo{year}{2023}\natexlab{}.
\newblock \showarticletitle{A Mathematical Runtime Analysis of the
  Non-dominated Sorting Genetic Algorithm {III} {(NSGA-III)}}. In
  \bibinfo{booktitle}{\emph{Proceedings of the International Joint Conference
  on Artificial Intelligence, {IJCAI}~2023}}. \bibinfo{publisher}{ijcai.org},
  \bibinfo{pages}{5657--5665}.
\newblock


\bibitem[Wietheger and Doerr(2024)]%
        {DoerrNearTight}
\bibfield{author}{\bibinfo{person}{Simon Wietheger} {and}
  \bibinfo{person}{Benjamin Doerr}.} \bibinfo{year}{2024}\natexlab{}.
\newblock \showarticletitle{Near-Tight Runtime Guarantees for Many-Objective
  Evolutionary Algorithms}. In \bibinfo{booktitle}{\emph{Proceedings of the
  International Conference of Parallel Problem Solving from Nature}}
  \emph{(\bibinfo{series}{PPSN XVIII})}. \bibinfo{publisher}{Springer},
  \bibinfo{pages}{153–168}.
\newblock


\bibitem[Zheng and Doerr(2022)]%
        {DoerrApprox}
\bibfield{author}{\bibinfo{person}{Weijie Zheng} {and}
  \bibinfo{person}{Benjamin Doerr}.} \bibinfo{year}{2022}\natexlab{}.
\newblock \showarticletitle{Better approximation guarantees for the NSGA-II by
  using the current crowding distance}. In
  \bibinfo{booktitle}{\emph{Proceedings of the Genetic and Evolutionary
  Computation Conference}} \emph{(\bibinfo{series}{GECCO 2022})}.
  \bibinfo{publisher}{ACM Press}, \bibinfo{pages}{611–619}.
\newblock


\bibitem[Zheng and Doerr(2024a)]%
        {Zheng2023Inefficiency}
\bibfield{author}{\bibinfo{person}{Weijie Zheng} {and}
  \bibinfo{person}{Benjamin Doerr}.} \bibinfo{year}{2024}\natexlab{a}.
\newblock \showarticletitle{Runtime Analysis for the NSGA-II: Proving,
  Quantifying, and Explaining the Inefficiency for Many Objectives}.
\newblock \bibinfo{journal}{\emph{IEEE Transactions on Evolutionary
  Computation}} \bibinfo{volume}{28}, \bibinfo{number}{5}
  (\bibinfo{year}{2024}), \bibinfo{pages}{1442--1454}.
\newblock


\bibitem[Zheng and Doerr(2024b)]%
        {Zheng_Doerr_2024}
\bibfield{author}{\bibinfo{person}{Weijie Zheng} {and}
  \bibinfo{person}{Benjamin Doerr}.} \bibinfo{year}{2024}\natexlab{b}.
\newblock \showarticletitle{Runtime Analysis of the SMS-EMOA for Many-Objective
  Optimization}. In \bibinfo{booktitle}{\emph{Proceedings of the AAAI
  Conference on Artificial Intelligence}} \emph{(\bibinfo{series}{AAAI 2024})}.
  \bibinfo{publisher}{AAAI press}, \bibinfo{pages}{20874--20882}.
\newblock


\bibitem[Zheng et~al\mbox{.}(2022)]%
        {ZhengLuiDoerrAAAI22}
\bibfield{author}{\bibinfo{person}{Weijie Zheng}, \bibinfo{person}{Yufei Liu},
  {and} \bibinfo{person}{Benjamin Doerr}.} \bibinfo{year}{2022}\natexlab{}.
\newblock \showarticletitle{A First Mathematical Runtime Analysis of the
  Non-dominated Sorting Genetic Algorithm {II} {(NSGA-II)}}. In
  \bibinfo{booktitle}{\emph{Proceedings of the {AAAI} Conference on Artificial
  Intelligence}} \emph{(\bibinfo{series}{AAAI 2022})}.
  \bibinfo{publisher}{{AAAI} Press}, \bibinfo{pages}{10408--10416}.
\newblock


\bibitem[Zitzler(1999)]%
        {ZitzlerPhD}
\bibfield{author}{\bibinfo{person}{Eckart Zitzler}.}
  \bibinfo{year}{1999}\natexlab{}.
\newblock \bibinfo{booktitle}{\emph{Evolutionary Algorithms for Multiobjective
  Optimization: Methods and Applications}}.
\newblock \bibinfo{publisher}{Shaker Verlag GmbH}, \bibinfo{address}{Germany}.
\newblock


\end{thebibliography}

\end{document}